\documentclass[runningheads]{llncs}
\usepackage[T1]{fontenc}
\usepackage{graphicx}
\usepackage{booktabs}
\usepackage[misc]{ifsym}
\newcommand{\corr}{(\Letter)}
\usepackage{mwe}
\usepackage{algorithm}
\usepackage{algorithmic}
\usepackage{amsmath,amssymb}
\begin{document}

\title{High-Dimensional Bayesian Optimization via
Random Projection of Manifold Subspaces}

\titlerunning{HD-BO via Random Projection of Manifold Subspaces}
% If the full title of your paper is short enough to also fit in the running head, you can omit the abbreviated paper title here. You can check as follows: if you comment out the \titlerunning line, something will appear in the header of all odd-numbered pages of your PDF from page 3 onward. This something is either the full title (in which case all is well), or the error message "Title Suppressed Due to Excessive Length". If this error message appears, you're going to want to provide an abbreviated title within the \titlerunning command, because if you won't do it, Springer will do it for you.

%N.B.: Author information (both in the \author{} and \authorrunning{} command) should only be present in the Camera-Ready Version of your paper. The version that you initially submit for review, ought to be double-blind. So, when initially submitting your paper, use:
%\author{Author information scrubbed for double-blind reviewing}

\author{Quoc-Anh Hoang Nguyen\inst{1} \and The Hung Tran \corr \inst{2} }

%\author{Andr\'e Lauren Benjamin\inst{1} \and
%Calvin Cordozar Broadus Jr.\inst{2,3} \corr \and
%Antwan Andr\'e Patton\inst{1}\orcidID{0000-1111-2222-3333}}
% You may leave out the orcidID information, if you want to.
% Use \corr to indicate the corresponding author. Note the spacing around the \corr command. Only one author can be the corresponding author.

%N.B.: comment out the \authorrunning{} command for the double-blind version of your paper submitted for review. Later, if your paper is accepted, use the command for the Camera-Ready Version.
\authorrunning{Nguyen et al.}
% If there are two authors, write 'A.L. Benjamin and C.C. Broadus Jr.'
% If there are more than two authors, '[...] et al.' is used.

\institute{FPT Software AI Center, Vietnam \email{nhquocanh@gmail.com}
\and
HaNoi University of Science and Technology, HUST, Vietnam \email{hungtt@soict.hust.edu.vn}}

\tocauthor{Quoc-Anh~Nguyen,Hung~Tran}

\toctitle{High-Dimensional Bayesian Optimization via Random Projection of Manifold Subspaces}

\maketitle              % typeset the header of the contribution

\begin{abstract}
Bayesian Optimization (BO) is a popular approach to optimizing expensive-to-evaluate black-box functions. Despite the success of BO, its performance may decrease exponentially as the dimensionality increases. A common framework to tackle this problem is to assume that the objective function depends on a limited set of features that lie on a low-dimensional manifold embedded in the high-dimensional ambient space. The latent space can be linear or more generally nonlinear. To learn feature mapping, existing works usually use an encoder-decoder framework which is either computationally expensive or susceptible to overfitting when the labeled data is limited. This paper proposes a new approach for BO in high dimensions by exploiting a new representation of the objective function. Our approach combines a random linear projection to reduce the dimensionality, with a representation learning of the nonlinear manifold. When the geometry of the latent manifold is available, a solution to exploit this geometry is proposed for representation learning.
In contrast, we use a neural network. To mitigate overfitting by using the neural network, we train the feature mapping in a geometry-aware semi-supervised manner. Our approach enables efficient optimization of BO’s acquisition function in the low-dimensional space, with the advantage of projecting back to the original high-dimensional space compared to existing works in the same setting. Finally, we show empirically that our algorithm outperforms other high-dimensional BO baselines in various synthetic functions and real applications.

\end{abstract}

\section{Introduction} \label{sec:intro}
Bayesian optimization (BO) has emerged as a powerful optimization framework to globally maximize expensive and noisy black-box functions. Thanks to its ability to model complex noisy cost functions in a data-efficient manner, BO has been successfully applied in a variety of applications ranging from hyperparameters tuning for machine learning algorithms \cite{NIPS2012_05311655} to the optimization of parametric policies in challenging robotic scenarios \cite{8461237}. However, it is known that BO is exponentially difficult with dimension. BO requires optimizing an acquisition function at each iteration which is a non-convex optimization problem in the same original search space. In high dimensions, optimizing the acquisition functions is really expensive. In addition, inaccurate solutions at the acquisition function optimization step may significantly affect the efficiency of BO algorithms.

Fortunately, in many situations, most dimensions have little impact on the objective function \cite{JMLR:v13:bergstra12a}. Therefore, we can assume that the objective function has effective dimensions that lie on a lower-dimensional manifold. For example, \cite{Carlsson2009TopologyAD} reveals that a large number of $3 \times 3$ images represented as points in $\mathbb{R}^9$ approximately lie on a two-dimensional Klein bottle manifold. As a result, if we apply function $f$ to these $3 \times 3$ images, the function $f$ possesses $2$  ``intrinsic'' degree of freedom and has an effective 2-dimensional Klein bottle manifold.
In these cases, the function is considered as having a lower intrinsic dimensionality. Under such an assumption, a high-dimensional optimization problem can be transformed into a low-dimensional optimization problem. 
Previous studies \cite{Wang2016BayesianOI,Nayebi2019AFF,NEURIPS2020_10fb6cfa} employ random linear embedding for dimensionality reduction, followed by optimization in the embedded subspace with lower dimensionality. These methods leverage the distance-preserving property of random projection, as stated in the Johnson-Lindenstrauss lemma \cite{Johnson1984ExtensionsOL}. On the other hand, \cite{Eriksson2021HighDimensionalBO} argues that the choice of an appropriate surrogate model is critical for achieving good performance. They propose SAASBO, which utilizes sparse priors on the Gaussian Process (GP) length scales. However, their work is also restricted only to a class of functions having an effective linear subspace. 

Recent works have been extended for a more general class of functions that have an effective nonlinear subspace.
%\cite{Wang2016BayesianOI} propose the REMBO algorithm using a random linear projection. However, REMBO performs poorly even for some synthetic problems due to the over-exploration of boundaries and distortion in embedding. HesBO \cite{Nayebi2019AFF} solves these drawbacks by using two hash functions to generate the inverse subspace embedding in order to retrieve the vector in the original space from the low dimensional vectors. Another work proposed by \cite{ijcai2019p596} (SIRBO) directly introduces a supervised dimension reduction method SIR to automatically learn the intrinsic structure of the objective function. \cite{45016f41f3844a1087845a2b542f5da9} propose selecting the projection dimensions by using nested random subspaces combined with the trust region method. When the active subspace is axis-aligned, SAASBO \cite{Eriksson2021HighDimensionalBO} uses sparse priors on the GP length scales. However, the cost of inference scales cubically with the number of function evaluations. 
%Taking the advantage of SIR method, SILBO [\cite{chen2020semi}] learns the projection matrix iteratively through the semi-supervised method.  
%Note that all these methods are restricted only to a class of functions having an effective linear subspace. Recent works have been extended for a more general class of functions that have an effective nonlinear subspace.
\cite{GmezBombarelli2018AutomaticCD,Notin2021ImprovingBO} propose using the variational autoencoders (VAE) to learn manifold subspaces. The key insight of using VAE as a nonlinear embedding is that it decodes from the embedded low-dimensional space to the original high-dimensional space. However, these approaches cannot exploit the geometry property of latent space when it is available, which can lead to suboptimal latent representation.
%However, this approach requires learning the embedding without the ability to update the learned feature space during the optimization. Therefore, it cannot utilize the valuable knowledge of labeled data to actively learn a latent space for BO.
\cite{moriconi2020high} proposes an encoder-decoder approach for BO but in an online fashion without using unlabeled data. They use a multi-layer neural network (NN) to build an encoder via supervised learning. However, this raises the overfitting issue when the labeled data is scarce. Our approach leverages the power of semi-supervised learning to mitigate these challenges. Besides that, this work uses a multi-output GP for the reconstruction scheme, which is expensive in practice even with its sparse variant. Consequently, their experiments were made with less than 60 dimensions. In contrast, our reconstruction scheme uses only a multiplication of matrix and thus is scalable to very high dimensions $(D \geq 500)$. Additionally, in contrast with the above methods, our method can efficiently exploit the geometry of the manifold when it is available. Compared to current works in \cite{NEURIPS2020_f05da679,Jaquier2021GeometryawareBO} which use geometry awareness to construct projections, our setting is not restricted by a constraint that the low-dimensional effective manifold inherits the geometry of the original manifold.

In this paper, we focus on the high-dimensional Bayesian optimization (HD-BO) problem under the assumption that the objective function $f$ has an effective low-dimensional manifold subspace. The main contributions of this work are as follows:
\begin{itemize}
    \item We propose an effective algorithm for HD-BO based on a new representation of the objective function. Our algorithm uses a random projection to reduce the dimensionality combined with a representation learning of the latent space. Our algorithm with theoretical support can avoid the disadvantages of existing works in the same setting. 
    \item We also propose a novel geometry-aware consistency loss to train the model in a semi-supervised manner to reduce the overfitting issues.
    %\item We propose a simple approach for enhancing the optimization of acquisition function in low-dimensional space supported by theoretical validation.
    \item We show empirically that our algorithm outperforms other high-dimensional BO approaches via both linear and non-linear embeddings in various synthetic functions and real applications.
\end{itemize}

\section{Problem Setting}
We consider the global maximization problem of the form
\begin{align}
    x^* = \text{argmax}_{x \in \mathcal X} f(x) 
\end{align}
in a compact search space $\mathcal X = [-1, 1]^D$. In this paper, we are especially concerned about problems with high values of $D$. We consider function $f$ that is black-box and expensive
to evaluate, and our goal is to find the optimum in a minimal number of samples. We further assume that we only have access to noisy evaluations of $f$ in the form $u = f(x) + \epsilon$, where the noise $\epsilon \sim \mathcal N (0, \sigma^2)$ is i.i.d. Gaussian distribution.

To avoid the curse of dimensionality, a common framework is to assume that the objective function depends on a limited set of features that lie on a lower-dimensional manifold embedded in the high-dimensional ambient space. We assume that this latent space can be identified as a $d$-dimensional manifold $\mathcal M$, where $d<< D$. We denote $S_{\mathcal{M}}(x) = \{s \in \mathbb{R}^D \mid s \in \operatorname{argmin}_{m \in \mathcal{M}}||m - x||\}$.  If $|\mathcal{S}_{\mathcal{M}}(x)|=1 \forall x \in \mathcal{X},$ there exists a mapping $P_{\mathcal{M}}: \mathcal{X} \subset \mathbb{R}^D \rightarrow \mathcal{M} \subset \mathbb{R}^D$, which outputs the nearest point on manifold $\mathcal{M}$. We call this mapping the orthogonal projection to manifold $\mathcal{M}$.
We say that the function $f:\mathcal{X}\rightarrow\mathbb{R}$ has the effective latent manifold $\mathcal{M}$ iff $\forall x \in \mathcal{X}, \exists x_{\mathcal{M}} \in \mathcal{S}_{\mathcal{M}}(x): f(x) = f(x_{\mathcal{M}}) $.
We note that this definition is the generalization of that in \cite{Wang2016BayesianOI} and the following works \cite{ijcai2019p596,chen2020semi} which is designed for the functions having a linear latent manifold. 

\subsection{Bayesian Optimization}
BO belongs to the class of sequential-based optimization algorithms. Following the routine of BO, there are two main steps that need to be specified: surrogate function and acquisition function. Gaussian process \cite{RASMUSSEN2005} is a popular choice for the surrogate function due to its tractability for posterior and predictive distribution. Given $t$ pairs of observation points $\mathbf{x_{1:t}}$ and the corresponding evaluation $\mathbf{y_{1:t}} = f(\mathbf{x_{1:t}}) + \epsilon$, where $\epsilon \sim \mathcal{N}(0,\sigma^2)$. We have a joint distribution follow a multivariate normal distribution $f(\mathbf{x_{1:t}}) \sim \mathcal{N}(\mu(\mathbf{x_{1:t}}), K(\mathbf{x_{1:t}}, \mathbf{x_{1:t}}))$ where $\mu(x)$ is the mean function and $K(\mathbf{x_{1:t}}, \mathbf{x_{1:t}})$ is the covariance matrix based on the kernel function $k$. Popular covariance functions include Mat\'ern kernel or squared exponential kernel etc. Given the new observation point $x*$, the posterior predictive distribution can be derived as: $f(x*)|x*,\mathbf{x_{1:t}},f(\mathbf{x_{1:t}}) \sim \mathcal{N}(\mu_{t+1}(x*), \sigma_{t+1}(x*))$ where $\mu_{t+1}(x*) = \mathbf{K_*^T} (\mathbf{K}+\sigma^2 \mathbf{I})^{-1}f(\mathbf{x_{1:t}}) + \mu(x*), \sigma_{t+1}(x*) = k(x*,x*) - \mathbf{K_*^T} (\mathbf{K}+\sigma^2 \mathbf{I})^{-1}\mathbf{K_*}$. In the above expression, we define $\mathbf{K_*} = [k(x*,x_1),...,k(x*,x_t)]^T, \mathbf{K} = K(\mathbf{x_{1:t}}, \mathbf{x_{1:t}})$

The acquisition functions are designed to balance exploration and exploitation at each iteration. Exploration means that we find candidates in an uncertain region whereas exploitation means that we find candidates in a highly probable region. Some examples of acquisition functions include Expected Improvement (EI) \cite{10.1007/3-540-07165-2_55} and GP-UCB \cite{6138914}. An EI acquisition function at iteration $t+1$ is defined as:
\begin{equation}
    \alpha_{t+1}(x) = \sigma_{t+1}(x)\text{pdf}(u) + [\mu_{t+1}(x) - f^*]\text{cdf}(u)
\end{equation}
where $\mu_{t+1}(x)$ is the GP posterior mean, $\sigma_{t+1}(x)$ is the GP posterior variance, $f^* = \max_{i= \overline{1,t}}y_i$, $u = \frac{\mu_{t+1}(x)-f^*}{\sigma_{t+1}(x)}$, pdf(.) and cdf(.) are the standard normal p.d.f and c.d.f, respectively.  
      
\section{Proposed Approach}  \label{sect:Proposed} 
In this section, we present a novel approach for HD-BO. We exploit the effective low dimensional manifold of the objective function for BO by expressing the objective function $f: \mathbb{R}^D \rightarrow \mathbb{R}$ as a composition of an orthogonal projection $h: \mathcal X \rightarrow \mathcal M$, 
where $\mathcal M$ is the $d$-dimensional manifold embedded in $\mathbb{R}^D$ by assumption, and a function $g: \mathcal{M} \rightarrow \mathbb{R}$ so that $f = g \circ h$. Here, $g$ and $h$ are both unknown.  
Since $g$ is still an expensive high-dimensional function, we continue to analyze $g$ as a composition of a projection $\Phi: \mathcal{M} \rightarrow \mathbb R^m$ and a lower dimensional function $g_0: \mathbb R^m \rightarrow \mathbb{R}$, where $d \le m << D$ is a parameter to be defined. As we will discuss below, such a composition of the objective function allows us to exploit the geometrical properties of the manifold and the existing results in the manifold learning field. This composition allows the optimization function to be optimized in the low-dimensional space $\mathbb R^m$ instead of the original high-dimensional parameter space. In the next sections, we will discuss how to represent the mapping $h$, the projection $\Phi$, and, finally propose our BO algorithm.       

%On the other hand, our composition does not need the inheritance of latent space, which is a strict assumption for the construction of composition in \cite{NEURIPS2020_f05da679}. \\

\subsection{Representing the mapping $h$} 
\label{subsec: learning_h}
\subsubsection{Geometry-aware representation $h$}
\begin{figure}
    \centering
    \includegraphics[width=0.8\linewidth]{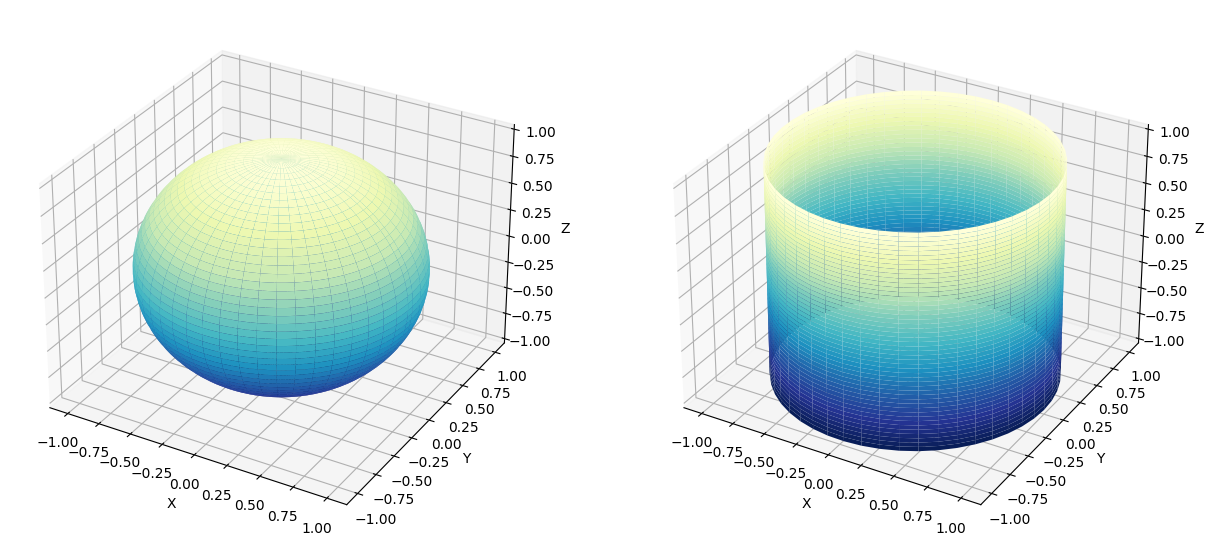}
    \caption{Left: The spherical $2-$dimensional manifold $\mathcal{S}^2$ embedded in $\mathbb{R}^3$. Right: The mixed $2-$dimensional manifold $\mathbf{M}^2$ embedded in $\mathbb{R}^3$ with $(x,y) \in \mathcal{T}^1$ and $z \in R$ .}
    \label{fig:manifold}
\end{figure}
When the geometry of the manifold is known, we can represent analytically the formulation of the projection mapping $h$ to leverage the geometry property of the latent manifold. The visualization of the geometry of some latent manifolds can be seen in Figure \ref{fig:manifold}. Some special types we can easily exploit are spherical geometry, linear geometry, etc. For the linear manifold, which is the most used assumption in the existing works, we can construct the mapping $h$ as:
\begin{equation} \label{eqa:h_linear}
    h(x) = h_{\mathbf{B}}(x) = \mathbf{B}\mathbf{B}^Tx
\end{equation}
where column space of matrix $\mathbf{B} \in \mathbb{R}^{D \times d}$ is a basis of effective subspace. Therefore, the parameter of feature mapping $h$ is matrix $\mathbf{B}$. 

Or for the nonlinear manifold such as spherical geometry, the objective function $f$ can be expressed as:
%We can build a function $f$ as follows.
\begin{equation}  \label{eqa:sphere}
    f(x_1,x_2,...,x_D) = g(z_1,z_2,...,z_{d+1})
\end{equation}
where $z_i = \frac{x_i}{\sqrt{\sum_{i=1}^{d+1} x_i^2}} \quad \forall i=\overline{1,d+1}$.

The target function $f$ depends on only the first $(d+1)$ dimensions which lie on the $d$-dimensional sphere $\mathcal{S}^d$. %The feature mapping $h$ is the orthogonal projection from the original search space to the effective $\mathcal{S}^d$.
In this case, we can construct the mapping $h$ as:
\begin{equation} \label{eqa:h_sphere}
    h(x) = h_{\mathbf{B},r,c}(x) = r\frac{\mathbf{B}(\mathbf{B}^Tx-c)}{||\mathbf{B}(\mathbf{B}^Tx-c)||_2} + \mathbf{B}c
\end{equation}
where the column space of $\mathbf{B} \in \mathbb{R}^{D \times (d+1)}$ is the basis of an affine subspace containing the effective latent sphere $\mathcal{S}^d$ with radius $r \in \mathbb{R}$ and centroid $c \in \mathbb{R}^{d+1}$.
%To build the formulation of $h$, we observe that $\mathcal{M}$ lie in linear subspace $L$ of dimension $d+1$. Therefore, we can build the mapping by projecting $x \in \mathcal{X}$ to linear subspace $L$. After that, in the linear subspace $L$, we project onto the $d-$dimensional sphere $\mathcal{S}^d$. Since we do not know $L$ as well as $\mathcal{S}^d$, we can reparameterize it. Specifically, let $B \in \mathbb{R}^{D \times (d+1)}$ is a matrix where its column space is the basis of a linear subspace $L$; $r \in \mathbb{R}$ and $c \in \mathbb{R}^{d+1}$ are the radius and the centroid of spherical manifold $\mathcal{S}^d$, respectively. Therefore, we can construct the mapping as:
%\begin{equation} \label{eqa:h_sphere}
%    h(x) = h_{B,r,c}(x) = r\frac{B(B^Tx-c)}{||B(B^Tx-c)||_2} + Bc
%\end{equation}
The learned parameters are $\{\mathbf{B}, r, c\}$. 
\subsubsection{Geometry-unaware representation $h$}
In general, when the geometry of latent manifolds is unknown, we use a multiple-layer neural network (NN) to model $h$ as in previous works of \cite{inproceedings,moriconi2020high}. Neural networks have already been applied successfully for modeling non-smooth responses in robot locomotion \cite{inproceedings}, and have also proven useful for learning the orientation of images from high-dimensional images \cite{wilson16}.
\subsubsection{Learning $h$}
To learn $h,$ we learn jointly $h$ and the high-dimensional representation $g$ of $f$ using the same supervised loss function as stated in \cite{inproceedings}. Particularly, we define a variant of manifold Gaussian Process (mGP) \cite{inproceedings} under our new composition so that $f \sim \mathcal{GP} (\mu_m, k_m)$ with mean function $\mu_m: \mathcal X \rightarrow \mathbb{R}$ and covariance function $k_m: \mathcal X \times \mathcal X \rightarrow \mathbb{R}$ defined as $\mu_m(x) = \mu(h(x))$ and $k_m(x_i, x_j) = k(h(x_i), h(x_j))$, with $\mu: \mathcal M \rightarrow \mathbb{R}$ and $k: \mathcal M \times \mathcal M \rightarrow \mathbb{R}$ a kernel function. The supervised loss function is the negative marginal likelihood:
\begin{equation} \label{super_loss}
    L_s(\theta_h, \theta_g) = -p(y| X, \theta_h, \theta_g)
\end{equation}
where $\theta_h$ and $\theta_g$ are the parameters of $h$ and $g$, respectively; $y$ is the noisy observation of the objective function $f$. Although mapping $h$ can be trained in a supervised manner, it is prone to overfitting due to the lack of labeled training data at the first few iterations of the algorithm.

To overcome this, we take the view of semi-supervised learning (SSL) \cite{Thomas2009}, which alleviates the necessity for labeled data by enabling the model to utilize unlabeled data. SSL has shown much promise in improving machine learning models when the labeled data is scarce \cite{Antti2017,Samuli2017,Xie2020}. One common approach of SSL is the use of consistency loss \cite{Xie2020,Berthelot2019} on a large amount of unlabeled data to constrain model predictions to be noise-invariant. In this work, we introduce a novel unsupervised consistency loss that can leverage the property of manifold. Indeed, we have the following property:
\begin{proposition} \cite{Leobacher2021} \label{prop:orthogonal_line}
    Let manifold $\mathcal{M} \subset \mathbb{R}^D$. Let $x \in \mathbb{R}^D$ arbitrary and $x_{\mathcal{M}} \in \mathcal{S}_{\mathcal{M}}(x)$. Then $\forall \lambda \in [0,1)$ , $P_{\mathcal{M}}(\lambda x + (1 - \lambda)x_{\mathcal{M}}) = x_{\mathcal{M}}.$ 
\end{proposition}
Intuitively, proposition \ref{prop:orthogonal_line} states that the orthogonal projection of all points lies in the line segment between $x$ and $x_{\mathcal{M}}$ is exactly $x_{\mathcal{M}}$. Therefore, to learn the orthogonal projection $h$, we can use the discretization technique and construct the following unsupervised consistency loss:
\begin{equation} \label{unsuper_loss}
    L_{us}(\theta_h) = \frac{1}{pq}\Sigma_{j=1}^p\Sigma_{i=1}^q \|h(\lambda_j x_i' + (1 - \lambda_j)h(x_i')) - h(x_i')\|_2
\end{equation}
where $\theta_h$ is the parameter of $h$, $\lambda_j \sim \operatorname{Uniform}(0,1) \forall j = \overline{1,p}$ and $x'_i  \sim \operatorname{Uniform}(\mathcal{X}) \forall i = \overline{1,q}$. The combined loss $L$ for semi-supervised learning mapping $h$ is defined as:
\begin{equation} \label{semisuper_loss}
    L(\theta_h, \theta_g) = L_s(\theta_h, \theta_g) + \gamma L_{us}(\theta_h)
\end{equation}
where $L_s$ is from Equation (\ref{super_loss}), $L_{us}$ is from Equation (\ref{unsuper_loss}), $\gamma$ is a weighting factor to balance the supervised loss and the unsupervised consistency loss. With geometry-aware representation at Equations (\ref{eqa:h_linear}) and (\ref{eqa:h_sphere}), the unsupervised consistency loss $L_{us}(\theta_h) = 0$. Therefore, we do not need unlabeled data to train the mapping $h$. \textbf{We will show that training with semi-supervised loss function in Equation (\ref{semisuper_loss}) can reduce the overfitting issue in Appendix D.4.}

\subsection{Representing the projection $\Phi$} \label{sect:contraction}
Defining the projection $\Phi: \mathcal{M} \rightarrow \mathbb R^m$ is important because it is related to the regression process of the function $g$, and thus of the objective function $f$. 
Following the existing work of the manifold learning \cite{10.1214/15-AOS1390}, to achieve a convergence guarantee of the regression of $g$, we need a condition that the reflection of $\Phi$ on the manifold $\mathcal{M}$ is a diffeomorphism. We propose using a random projection via a random orthogonal matrix $\mathbf{A}$, which can preserve the pairwise metric distance of sample points on the manifold with high probability. Let $\mathbf{A}$ be a random orthogonal matrix from $\mathbb{R}^D$ to $\mathbb{R}^m$ with $m$ satisfies $m = \mathcal O(d\text{log}(D))$ \cite{Baraniuk2009RandomPO}, then with the mild condition of manifold $\mathcal{M}$, with high probability, the following relationship holds on every point $x_{\mathcal{M}}, y_{\mathcal{M}} \in \mathcal{M}$:
\begin{equation} \label{eqa:preserve}
    (1-\epsilon)\sqrt{\frac{m}{D}} \leq \frac{\|\mathbf{A}x_{\mathcal{M}}-\mathbf{A}y_{\mathcal{M}}\|_2}{\|x_{\mathcal{M}}-y_{\mathcal{M}}\|_2} \leq (1+\epsilon)\sqrt{\frac{m}{D}}
\end{equation}
By the work of \cite{JMLR:v17:14-230}, Equation (\ref{eqa:preserve}) implies that with high probability, the projection dimension $\mathbf{A}$ is a diffeomorphism onto its image. Due to the space limitation, in Appendix B, we will discuss in detail the diffeomorphism as well as this random projection.

\subsection{Proposed Bayesian Optimization Algorithm using Random Projection}
%The idea of jointly learning a feature mapping and a representation of the objective function with a mGP was also exploited in the context of high-dimensional BO \cite{moriconi2020high}. To learn the feature mapping, they use a multi-layer neural network with parameters as a representation of feature mapping $h$ and learn the presentation by a GP. However, unlike previous works, the representation $g$ of $f$ is in high dimensions in our setting. Hence, instead of learning $g$ using a high-dimensional GP, we propose to learn $g$ by learning a dimensionality reduction map from the original ambient space $\mathcal X$ to a feature space $\mathbb{R}^m$ using a random projection.
In this subsection, we present our new Bayesian optimization algorithm using the random projection we defined above. 

We construct a GP $g_0 \sim \mathcal{GP}(\mu(.), k_a(||.||) =\exp\{-a^2||.||^2\} )$ with the training data $\{\mathbf{A}h(x_i), f(x_i) + \epsilon_i\}_{i=1}^n$ and $a \sim Ga(a_0,b_0)$. Note that in practice, \cite{JMLR:v17:14-230} demonstrated that replacing the powered gamma prior for $a$ with a gamma prior has no effect on the results. As we will show in the Appendix B.1, $g_0$ can achieve an optimal convergence rate, and importantly we have constructed a GP in lower dimensional space $\mathbb{R}^m$. Therefore, we can efficiently optimize the acquisition in the low-dimensional space according to GP $g_0$.
\paragraph{Selecting next points and back projection scheme.}
We define $\mathcal{M}_{A} = \{z \in \mathbb{R}^m | \exists x \in \mathcal{M} \subset \mathbb{R}^D: z = \mathbf{A}x\}$. Intuitively, $\mathcal{M}_A$ is an image of manifold $\mathcal{M}$ through a linear embedding $\mathbf{A}$. By leveraging the distance-preserving property of matrix $\mathbf{A}$, we can ensure a convergence of the lower dimensional GP defined on $\mathcal{M}_A$ $g_0 \sim \mathcal{GP}(\mu(.), k_a(||.||) =\exp\{-a^2||.||^2\} )$ with the training data $\{\mathbf{A}h(x_i), f(x_i) + \epsilon_i\}_{i=1}^n$ and $a \sim Ga(a_0,b_0)$. 

A direct way to select the next points is to use the traditional EI acquisition function on low-dimensional space $\mathbb{R}^m$. However, we cannot ensure that these selected points lie on $\mathcal{M}_{A}$. Therefore, a challenge in this step is to select points on  $\mathcal{M}_{A}$ by solving the following problem:
\begin{equation} \label{eqa:EI_hard}
    z_{n+1} = \operatorname{argmax}_{z \in \mathcal{M}_A} \operatorname{EI}(z|\{\mathbf{A}h(x_i),y_i\}_{i=1}^n)
\end{equation}
Solving problem \ref{eqa:EI_hard} is hard since the domain $\mathcal{M}_A$ is complicated. To simplify problem (\ref{eqa:EI_hard}), 
%we instead solve the following problem:
%\begin{equation} \label{eqa:EI_easy}
%    z_{n+1} = \operatorname{argmax}_{z \in \mathbb{R}^m} \operatorname{EI}(\mathbf{A}h(\mathbf{A}^Tz)|\{\mathbf{A}h(x_i),y_i\}_{i=1}^n)
%\end{equation}
%We will show the problem (\ref{eqa:EI_hard}) and (\ref{eqa:EI_easy}) is equivalent when $h(.)$ approximates $P_{\mathcal{M'}}(.)$. 
we have the following theorems.
%\begin{lemma} \label{lemma:rank}
%    Let $\mathbf{O} \in \mathbb{R}^{n \times k}$ is the orthogonal matrix (i.e $\mathbf{O}^T\mathbf{O}=\mathbb{I}_{k \times k}$). Let $\mathbf{A} \in \mathbb{R}^{n \times s}$ be the random orthogonal matrix. Suppose $k \leq s < n$. Then with probability 1, $\operatorname{rank}(\mathbf{O}^T\mathbf{A}) = k$.
%\end{lemma}
%\begin{proof}
%    Since $\mathbf{A} \in \mathbb{R}^{n \times s}$ is the random orthogonal matrix, $\mathbf{A}$ uniformly distributed on the Stiefel manifold $\mathbb{V}_s(\mathbb{R}^n)$. Therefore, followed by Theorem 2.2.1 in \cite{Chikuse2003StatisticsOS}, $\mathbf{A}$ can be expressed as:
%    \begin{equation}
%        \mathbf{A} = Z(Z^TZ)^{-1/2}
%    \end{equation}
%    where $Z \in \mathbb{R}^{n \times s}$ is the random matrix where elements of $Z$ are i.i.d from normal distribution $\mathcal{N}(0,1)$.\\
%    We have:
%    \begin{align*}
%        \mathbf{O}^T\mathbf{A} = \mathbf{O}^TZ(Z^TZ)^{-1/2}
%    \end{align*}
%    Since $(Z^TZ)^{-1/2} \in \mathbb{R}^{m \times m}$ is invertible, $\operatorname{rank}(\mathbf{O}^TZ(Z^TZ)^{-1/2}) = \operatorname{rank}(\mathbf{O}^TZ)$\\
%    Moreover, followed by Lemma 2 in \cite{Wang2016BayesianOI}, with probability 1, $\operatorname{rank}(\mathbf{O}^TZ) = k.$
%\end{proof}
\begin{theorem} \label{theo:exists}
    Let $\mathbf{A} \in \mathbb{R}^{m \times D}$ be a random orthogonal matrix. Let $\mathcal{M} \subset \mathbb{R}^D$ be the $d-$dimensional manifold which can be embedded in Euclidean space $\mathbb{R}^m$. Then, with probability 1, $\forall x \in \mathcal{M}, \exists z \in \mathbb{R}^m: x = P_{\mathcal{M}}(\mathbf{A}^Tz).$
\end{theorem}
%\begin{proof}
%    A full proof is given in the Appendix. 
%\end{proof}
\begin{proof}
    The full proof is given in Appendix A.
\end{proof}
\begin{theorem}\label{prop:dual}
    Let $\mathcal{M} \subset \mathbb{R}^D$ is the $d-$dimensional manifold, which can be embedded in Euclidean space $\mathbb{R}^m$; $\mathbf{A} \in \mathbb{R}^{D \times m}$ is the random orthogonal matrix; $\mathcal{M}_A = \{z \in \mathbb{R}^m|\exists x \in \mathcal{M}: z = \mathbf{A}x\}$; $\overline{\mathcal{M}_A} = \{z \in \mathbb{R}^m|\exists q \in \mathbb{R}^m: z = \mathbf{A}P_{\mathcal{M}}(\mathbf{A}^Tq)\}$. Then, with probability 1, $\mathcal{M}_A=\overline{\mathcal{M}_A}$.
\end{theorem}
\begin{proof}
    Let $z \in \overline{\mathcal{M}_A} \rightarrow \exists q \in \mathbb{R}^m: z = \mathbf{A}P_{\mathcal{M}}(\mathbf{A}^Tq) \subset \mathcal{M}_A \rightarrow \overline{\mathcal{M}_A} \subseteq \mathcal{M}_A$.\\
    Let $z \in \mathcal{M}_A \rightarrow \exists x \in \mathcal{M}: z = \mathbf{A}x$. From Theorem \ref{theo:exists}, with probability 1, $\exists q \in \mathbb{R}^m: x = P_{\mathcal{M}}(\mathbf{A}^Tq) \rightarrow z = \mathbf{A}x = \mathbf{A}P_{\mathcal{M}}(\mathbf{A}^Tq) \rightarrow z \in \overline{\mathcal{M}_A} \rightarrow \mathcal{M}_A \subseteq \overline{\mathcal{M}_A}$ with probabilty 1. 
\end{proof}
Now with the probability 1, if $h$ approximates $P_{\mathcal{M}}$, we have:
\begin{align*}
    &\operatorname{argmax}_{z \in \mathcal{M}_A} \operatorname{EI}(z|\{\mathbf{A}h(x_i),y_i\}_{i=1}^n)\\
    = &\operatorname{argmax}_{z \in \overline{\mathcal{M}_A}} \operatorname{EI}(z|\{\mathbf{A}h(x_i),y_i\}_{i=1}^n)\\
    \approx &\operatorname{argmax}_{q \in \mathbb{R}^m; z = \mathbf{A}h(\mathbf{A}^Tq)} \operatorname{EI}(z|\{\mathbf{A}h(x_i),y_i\}_{i=1}^n)\\
    = &\operatorname{argmax}_{q \in \mathbb{R}^m} \operatorname{EI}(\mathbf{A}h(\mathbf{A}^Tq)|\{\mathbf{A}h(x_i),y_i\}_{i=1}^n)
\end{align*}
where the first equality happens due to theorem \ref{prop:dual} and the second equality happens due to the assumption $h$ approximates $P_{\mathcal{M}}$. Therefore, instead of solving problem \ref{eqa:EI_hard}, we will solve the below problem:
\begin{equation} \label{eqa:EI_easy}
    z_{n+1} = \operatorname{argmax}_{z \in \mathbb{R}^m} \operatorname{EI}(\mathbf{A}h(\mathbf{A}^Tz)|\{\mathbf{A}h(x_i),y_i\}_{i=1}^n)
\end{equation}
Note that we can optimize problem \ref{eqa:EI_easy} on the domain $[-\sqrt{m},\sqrt{m}]^m$ as stated in \cite{Wang2016BayesianOI}. After finding the solution to problem \ref{eqa:EI_easy}, we project it back to the high-dimensional space as $h(\mathbf{A}^Tz)$ for the evaluation.
Our back projection scheme is simple and convenient in contrast with heavy computing input reconstruction proposed by previous works \cite{moriconi2020high,NEURIPS2020_f05da679}. Our algorithm called the \textbf{RPM-BO} (\textbf{R}andom \textbf{P}rojection of \textbf{M}anifold-based \textbf{BO}), is presented in Algorithm 1.

\begin{algorithm}[tb]
 \caption{RPM-BO Algorithm}  \label{alg1}
 \textbf{Input:} $\mathcal{H} = \emptyset$;  initial points $\{x_1,x_2,...,x_s\} \subset \mathcal{X} \subset \mathbb{R}^D$; $a_0, b_0, \gamma, p, q \in \mathbb{R}$.\\
 \textbf{Output}: Final recommendation $x_N$.\\
 \begin{algorithmic}[1]
 \STATE Create $q$ random unlabeled data points $\{x_i'\}_{i=1}^q \in \mathcal{X}$.
 \STATE Create $p$ random coefficients $\{\lambda_j\}_{j=1}^p \in (0,1)$.
 \STATE Choose a suitable $m$ and create a random orthogonal matrix $\mathbf{A} \in \mathbb{R}^{m \times D}$. \\
 \STATE Define a parameterized mapping $h: \mathbb{R}^D \rightarrow \mathbb{R}^D$. \\
 \STATE Set $n:=s$
 \WHILE{$n \leq N$}
 \STATE Update the hyperparameters $\theta_h$ by minimize semi-supervised loss function in Equation (\ref{semisuper_loss}) with the history data $\mathcal{H}$.\\
 \STATE Construct a Gaussian Process $g_0(.) \sim \mathcal{GP}(\mu_0(.), k_a(.,.))$ where $k_a(z,z')=\exp\{-a^2||z-z'||^2\}$ and $a \sim Ga(a_0,b_0)$ with the training data $\{\mathbf{A}h(x_i), y_i\}_{i=1}^{n}$.\\
 \STATE Select the next query point by optimizing the acquisition function: $z_{n+1} = \operatorname{argmax}_{z \in [-\sqrt{m},\sqrt{m}]^m} \operatorname{EI} (\mathbf{A}h(\mathbf{A}^Tz))$.\\
 \STATE Project back to high-dimensional space: $x_{n+1} = h(\mathbf{A}^{\top} z_{n+1}$). \\
 \STATE Evaluate $y_{n+1} = f(x_{n+1}) + \epsilon_i$ \\
 \STATE Augment the set of observed data $\mathcal{H} = \mathcal{H} \cup (x_{n+1}, y_{n+1})$ \\
 \STATE $n := n + 1$ \\
 \ENDWHILE
 \end{algorithmic}
 \end{algorithm}

%\subsubsection{Dimensionality Reduction using Random Projection} \label{sect:m}
%In this subsection, we define a dimensionality reduction map $\Phi$ that can work for every type of manifold at the lowest cost. 
%The majority of manifold-modeled signal dimensionality reduction algorithms learn the manifold structure from a set of data points. Typically, this is done by creating nonlinear mappings from $\mathbb{R}^D$ to $\mathbb{R}^m$ for some $m<D$ that are adapted to the training data and intended to preserve some manifold property such as local linear neighborhood structure \cite{doi:10.1126/science.290.5500.2323} or the geodesics distance between all points \cite{doi:10.1126/science.290.5500.2319}. However, these approaches need to be retrained whenever the training data change, and especially the diffeomorphism property is not guaranteed. 
\subsection{Discussion} \label{sect:discuss}
\paragraph{The advantage of using random matrix $\mathbf{A}$.} First, the random matrix is easy to implement. Despite this simplicity, we can learn exactly the function $g$ via a low-dimensional GP regression as long as all data points suggested from $h$ lie on manifold $\mathcal{M}$ as we show in the Appendix B.1. Another advantage of using the random matrix is when data points $h(\mathcal{X})$ are not on $\mathcal{M}$, the random projection $\mathbf{A}$ can reduce the negative effects of noise on the convergence of the low-dimensional GP regression as mentioned in \cite{JMLR:v17:14-230}. This is an advantage of random projection compared to the other dimensionality reduction techniques, which can compress the noise of the near-manifold data points. 

\paragraph{On the correctness of Theorem \ref{prop:dual}.} Theorem \ref{prop:dual} holds when the $d$-dimensional manifold $\mathcal{M}$ can be embedded in Euclidean $\mathbb{R}^m$. We define $e(\mathcal{M})$ is the smallest integer that $\mathcal{M}$ can be embedded in $\mathbb{R}^{e(\mathcal{M})}$. According to Whitney embedding theorem \cite{Persson2014TheWE}, we have $e(\mathcal{M}) \leq 2d$. Therefore, if $m \geq 2d$, which is much lower than the original high-dimensional $D$, then every $d-$dimensional manifold can be embedded in $\mathbb{R}^m$. Moreover, $e(\mathcal{M})$ depends on the structure of the manifold and can be lower than the result of Whitney. For example, $e(\mathcal{S}^d) = d+1$ where $\mathcal{S}^d$ is the $d-$dimensional sphere.
\paragraph{On the choice of the parameter $m$.} The embedded dimension $m$ also needs to ensure the distance-preserving property stated in Equation (\ref{eqa:preserve}). In practice, we choose $m$ as large as possible, which is a common approach in random projection-based BO \cite{Wang2016BayesianOI,Nayebi2019AFF,NEURIPS2020_10fb6cfa}. A test on different values of $m$ is also performed in Section D.2 of the Appendix.

\paragraph{On the convergence of the proposed RPM-BO algorithm}
We discuss this problem in Section B.3 of the Appendix.

\section{Experiments}
\begin{figure}[h!]
    \centering
    \includegraphics[width=1.0\linewidth]{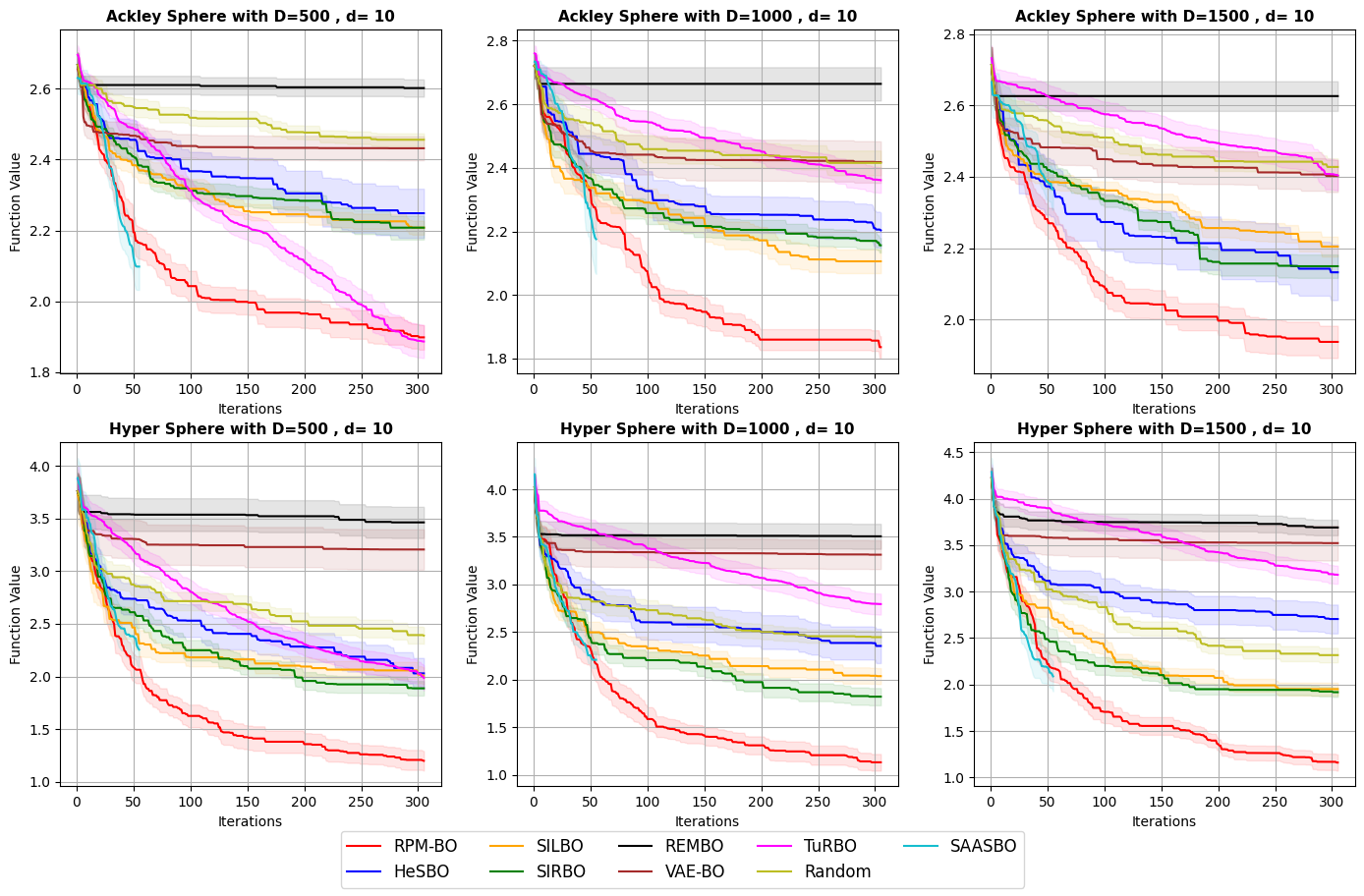}
    \caption{Performances on two standard functions with effective spherical manifold for 500, 1000, and 1500 input dimensions. For all cases, the dimension of the effective manifold is 10. The $y-$axis presents the value function (A smaller value is better).}
    \label{fig:sphere}
\end{figure}
\begin{figure}[h!]
    \centering
    \includegraphics[width=1.0\linewidth]{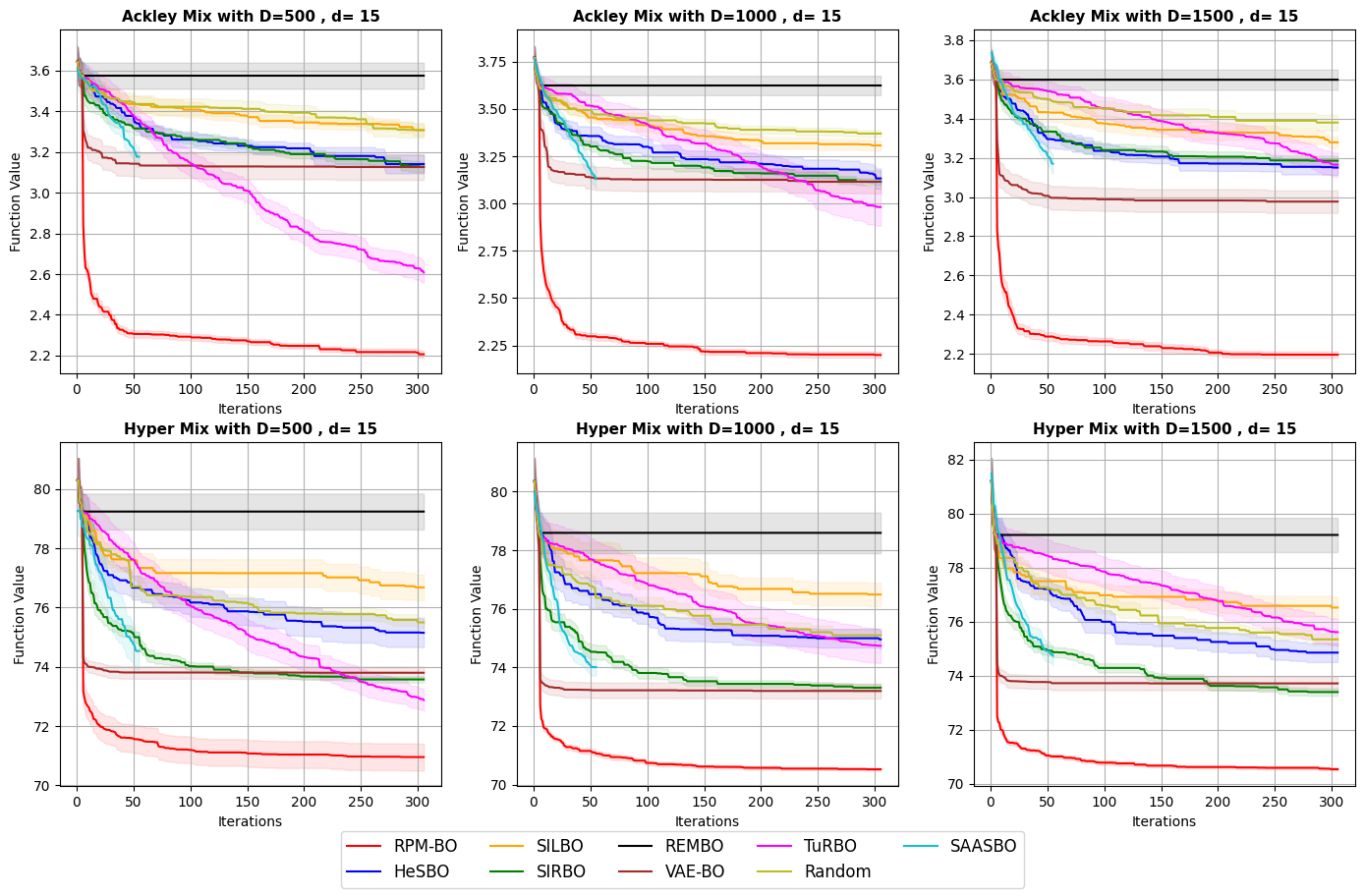}
    \caption{Performances on two standard functions with effective mixed-manifold for 500, 1000, and 1500 input dimensions. For all cases, the dimension of the effective manifold is 15. The $y-$axis presents the value function (A smaller value is better).}
    \label{fig:mixed}
\end{figure}
\begin{figure}[h!]
    \centering
    \includegraphics[width=1.0\linewidth]{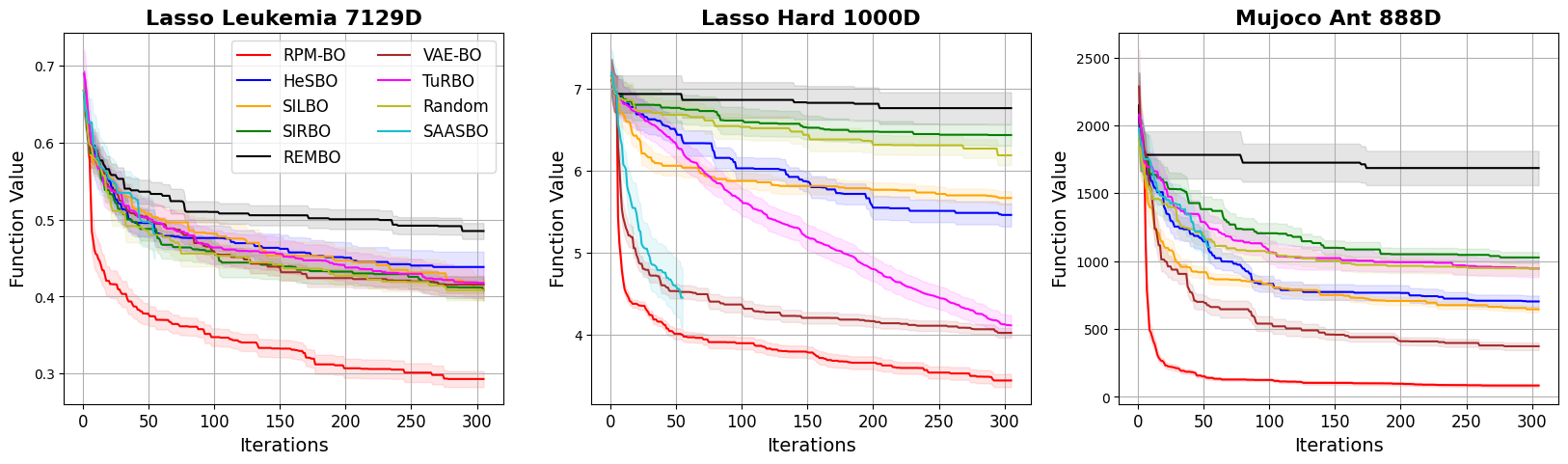}
    \caption{Performance on three real applications.}
    \label{fig:real}
\end{figure}
In this section, we evaluate the performance of our proposed method on a variety of synthetic functions that have different types of effective manifold subspace and on real-world problems. We compare our algorithm against eight baselines: REMBO \cite{Wang2016BayesianOI}; HeSBO \cite{Nayebi2019AFF}; SIRBO \cite{ijcai2019p596}; SILBO \cite{chen2020semi} ; VAE-BO \cite{Notin2021ImprovingBO}; TuRBO \cite{NEURIPS2019_6c990b7a}; random search method \cite{JMLR:v13:bergstra12a} and SAASBO \cite{Eriksson2021HighDimensionalBO}.  We use Expected Improvement as an acquisition function. For more detail about the settings, we refer the reader to section C in the Appendix.
\footnote{The code of RPM-BO is available at \url{https://github.com/Fsoft-AIC/RPM-BO}}
\subsection{Synthetic Functions} 
\subsubsection{Constructing function $f$ via the latent manifold subspace} \label{sect:def_man}
To construct a function $f$ having the effective low-dimensional manifold subspace, we start by constructing a manifold subspace. Unlike previous works using a linear manifold, we here construct a more complicated manifold. We use two types of manifolds. The first type has spherical geometry in section \ref{subsec: learning_h}, where we formulate the objective function shown in Equation (\ref{eqa:sphere}). For the second type, we will construct a more complicated manifold which is a mixture of a linear and a torus. Specifically, we define the objective function as:
\begin{equation} \label{eqa:mix}
    f(x_1,x_2,...,x_D) = g(z_1,z_2,...,z_{2d_1 + d_2})
\end{equation}
where 
$$
 \left\{\begin{array}{c}
z_{2i-1} = \frac{x_{2i-1}}{\sqrt{x_{2i-1}^2 + x_{2i}^2}} \quad \forall i=\overline{1,d_1}\\
z_{2i} = \frac{x_{2i}}{\sqrt{x_{2i-1}^2 + x_{2i}^2}}  \quad \forall i=\overline{1,d_1}\\
z_i = x_i \quad \forall i=\overline{(d_1 + 1),d_2} \\
\end{array}\right.
$$
The target function $f$ depends only on the first $(2d_1 + d_2)$ dimensions. However, the first $2d_1$ elements lie on the $d_1$-dimensional flat torus $\mathcal{T}^{2d_1}$ whereas the last $d_2$ elements lie on linear subspace $\mathbb{R}^{d_2}$. %Moreover, the Cartesian product of manifolds is also a manifold and the dimension of the product manifold is the sum of the dimensions of its factors.
Therefore, the function $f$ has an effective manifold represented by $\mathbf{M} = \mathcal{T}^{2d_1} \times \mathbb{R}^{d_2}$, which is the mix structure between a flat torus and a linear manifold. We call $\mathbf{M}$ "mixed-manifold". The dimension of $\mathbf{M}$ is $d_1 + d_2 =d$ so we can write $\mathbf{M}^d$. The visualization of $\mathbf{M}^2$ is shown in Figure \ref{fig:manifold} (Right).
%, while the objective function for the second case is represented by equation \ref{eqa:mix}. 
Due to the space limitation, we will add a linear $d-$dimensional manifold in Appendix D.5. For the feature mapping, we define the formulation of $h$ for the sphere in Equation (\ref{eqa:h_sphere}) while we construct a NN mapping for mixed manifold.

To build the above functions $f$, we use the following functions: Ackley and Rotated Hyper-Ellipsoid. Despite the initial input dimension being high, the effective latent manifold has a significantly lower dimension. This characteristic enables the convergence rate to match that of BO in a lower-dimensional space. Therefore, we ran each experiment 20 times with a total of 300 iterations and 10 initial points. As reported in \cite{Eriksson2021HighDimensionalBO}, SAASBO requires very high runtime and very large memory, thus we run SAASBO only within 60 iterations. Please see the runtime experiment in Appendix D.3.

\subsubsection{On the scalability of \textit{D}} \label{sect:5.1}
We evaluate the scalability of the proposed RPM-BO algorithm on three different values of dimension $D \in \{500,1000,1500\}$. We perform experiments on a variety of test functions with two types of latent manifold structure as discussed in \ref{sect:def_man}. For the effective manifold subspace $\mathcal{S}^d$, we set $d = 10$ and the projection dimension for all the other baselines is 10. On the other hand, for manifold subspace $\mathbf{M}^d$, we set $d_1 = 5, d_2 = 10$ ($d=15$) and the projection dimension is 15 for the other baselines. 

%\begin{figure*}
%    \centering
%    \includegraphics[width=1.0\linewidth]{test_m.pdf}
%    \caption{Performance on Ackley function with effective 10-dimensional sphere for varying projection dimension $m$.}
%    \label{fig:m}
%\end{figure*}

The result is shown in Figure \ref{fig:sphere} and Figure \ref{fig:mixed}. Note that we do not know the exact actual minimum, thus we will show function values instead. 
Overall, our method performs better than other baselines and is scalable to high dimensions. In contrast, TuRBO necessitates exploration in the high-dimensional space to reach the optimum, leading to a requirement for a larger number of iterations.
On the other hand, we can see that REMBO performs poorly in most cases. %This is partially caused by the fact that REMBO is limited to working with linear manifolds, which is a small case of manifolds in general.  
Additionally, both $\mathbf{M}^d$ and $\mathcal{S}^d$ depend on some sparse axis-aligned dimensions, which explains why SAASBO does well in some experiments such as the Hyper Sphere function with $D=500$. By heuristically learning the embedded effective manifold and feature mapping $h$, our method overcomes the limitation of the linear projection method and approximates the surrogate function in low-dimensional space via a compressed Gaussian Process. Although the VAE-BO baseline can be applied to the objective that has nonlinear embedded features, it requires a significant quantity of data to learn an offline meaningful reconstruction. However, the result of VAE-BO is not compared well with our method, which can learn the intrinsic subspace during the optimization. We further show the p-values to test the performance of our method in the Appendix D.1.

\subsection{Optimization on a real-world problem}
To further explore the performance of our method, we consider three real-world problems, including Lasso Hard, Lasso Leukemia benchmarks from LASSOBENCH \cite{https://doi.org/10.48550/arxiv.2111.02790} and MuJoCo Ant benchmark \cite{6386109}. %While Lasso Hard is built on a well-defined 1000-dimensional environment, the dataset of Lasso Leukemia includes 7129 gene expression values for predicting the type of Leukemia, thus creating a 7129-dimensional search space. Ant is robot locomotion tasks in MuJoCo, which is a popular black-box optimization benchmark. More real-world problems can be found in Appendix D.
\paragraph{LassoBench Problems}
LassoBench \cite{https://doi.org/10.48550/arxiv.2111.02790} is a high-dimensional benchmark for hyper-parameters optimization of Weighted Lasso Regression models. For Lasso Hard experiments, the dataset is drawn from a 1000-dimensional multivariate normal distribution with zero mean and unit variance. To study the robustness of observational noise, we test on a noisy variant of Lasso Hard. On the other hand, the Leukemia dataset includes 7129 gene expression values from 72 samples for predicting the type of Leukemia, which leads to a 7129-dimensional black-box optimization problem. Moreover, Lasso Hard and Lasso Leukemia pose 50, and 22 effective dimensionalities, respectively. Hence, for RPM-BO, we construct mapping $h$ shown in Equation (\ref{eqa:h_linear}).  %We set the projection dimension for all methods as 15. 
Figure \ref{fig:real} shows the validation loss for all methods. As seen in the figure, RPM-BO clearly outperforms the baselines since it can exploit the effective linear manifold subspace. Since the experiments pose sparse axis-aligned effective dimension, SAASBO can perform well on both Lasso Hard and Lasso Leukemia.
%\begin{figure}
%    \centering
%    \includegraphics[width=0.8\linewidth]{overfit.png}
%    \caption{Barplot of error on dimensionality reduction methods.}
%    \label{fig:overfit}
%\end{figure}
\paragraph{MuJoCo Locomotion Tasks}
We turn to the more difficult MuJoCo tasks in Reinforcement Learning. The goal is to find the parameters of a linear policy maximizing the accumulated reward. We set the objective function $f$ to be the negative of accumulated reward. The intuition for an effective latent manifold is that optimizing a subset of hinges, such as the 4 hinges connecting the leg links, while keeping the rest fixed, can allow the robot to move forward. We set the projection dimension for HeSBO, SIRBO, SILBO, REMBO is 10. For RPM-BO, since we do not know the geometry of latent space, the mapping $h$ is designed by a neural network from $\mathbb{R}^D \rightarrow \mathbb{R}^D$ %with single 35 units hidden layer, and as the activation function, we use sigmoid activation. To re-scale the output into the search space $\mathcal{X}$, we use tanh function. 
with $m=10$. The result is shown in Figure \ref{fig:real}. As we can see, our method achieves the best performance followed by VAE-BO. This shows the effectiveness of our method.
\section{Related work}
\paragraph{HD-BO with additive structure assumption} states that the objective function is decomposed into the sum of independent low-dimensional functions. Given the aforementioned structure, \cite{Kandasamy2015HighDB} provides an alternative acquisition function for additive kernels called ADD-GP-UCB which can be optimized over separately disjoint low-dimensional subspaces. \cite{Li2016HighDB} consider the additive model on the projected data that generalizes the additive assumption to the projected additive assumption. \cite{Wang2017BatchedHB} offer efficient methods for learning decomposition based on Markov Chain Monte Carlo (MCMC) methods. \cite{pmlr-v84-rolland18a} modify the additive model with disjoint decomposition to have overlapping groups, where the overlapping decomposition is represented by a dependency graph and the acquisition function is optimized using message passing protocol. Recent work shows that learning decompositions based on data can be easily misled towards local decompositions \cite{Ziomek2023}. Therefore, they proposed RDUCB, a data-independent approach to learning the decomposition using a random tree sampler with theory support. As opposed to all  those techniques, our method takes either a different assumption or direction to solve the HD-BO problem.

\paragraph{HD-BO with effective dimensionality assumption} \cite{ijcai2019p596} introduce the supervised dimension reduction method SIR to automatically learn the intrinsic structure of the objective function. However, this approach requires a substantial amount of initial data for learning the dimension reduction projection. To avoid this issue, \cite{chen2020semi} leverages a semi-supervised method to iteratively learn the projection matrix. Nevertheless, a limitation of linear embedding methods is the selection of the embedded dimension. \cite{45016f41f3844a1087845a2b542f5da9} propose using nested random subspaces in combination with the trust region method to choose the projection dimensions. It is important to note that previous works assuming effective dimensionality primarily focus on linear relationships between dimensions. Our approach can extend to capturing non-linear relationships among these dimensions.

\paragraph{HD-BO without structural assumptions} Scaling BO to high-dimensional without structural assumptions has been also investigated. To address this, \cite{Li2017HighDB} propose a dropout strategy where BO is applied to a randomly chosen subset of $d$ dimensions out of $D$ dimensions in each iteration. \cite{TranThe2020TradingCR} utilizes a discrete set of low-dimensional subspaces to optimize the acquisition function at a lower cost. To handle high-dimensional problems, the use of a trust region has shown promising results in reducing the search space size \cite{NEURIPS2019_6c990b7a}, although it may limit the exploration to a more local scale. \cite{Kirschner2019AdaptiveAS} addresses the challenge of optimizing non-convex acquisition functions in high dimensions by decomposing the global problem into a sequence of one-dimensional sub-problems, which can be efficiently solved. \cite{Dai2022} argue that using the GP as a surrogate model can limit the performance of BO on high-dimensional problems. They proposed using highly expressive neural networks as surrogate models followed by Thompson Sampling acquisition function, which avoids the need to invert the large matrix. Despite their capability to handle any type of objective function, these approaches must strike a balance between computational complexity and convergence rate in order to effectively solve the problem.

\paragraph{Manifold hypothesis in HD-BO}
Manifold is one of the most prominent examples of a non-Euclidean datatype involving a vast number of geometric surfaces such as curvy or other diverse shapes. Motivated by applications areas such as robotics \cite{Calinon2019GaussiansOR}, recent works have sought to generalize Gaussian Process and BO to the manifold setting \cite{Hutchinson2021VectorvaluedGP,Borovitskiy2020MaternGP}. \cite{inproceedings} propose Manifold Gaussian Processes (mGP), a novel supervised method that jointly learns a transformation of the data into a feature space and a Gaussian Process (GP) regression from the feature space. The mGP allows learning data representations, which are useful for the overall regression task. 
When the data is on a low-dimensional manifold, \cite{JMLR:v17:14-230} proposes a compressed Gaussian Process which uses a random projection matrix to construct a regression on lower dimension space.
\cite{Jaquier2019BayesianOM} propose GABO which use the geometry-aware kernels for building surrogate model when data lie on a Riemannian manifold.  To build a fully geometry-aware BO framework, they thereby optimize acquisition functions based on manifold optimization techniques. Although the result outperforms the classical BO in accuracy, GABO only works for low-dimensional input. \cite{NEURIPS2020_f05da679} scale up GABO to high-dimensional space called HD-GABO. HD-GABO features geometry-aware GP that jointly learns a nest-manifold embedding and a representation of an objective function in the latent space. However, HD-GABO requires that the latent space inherit the geometry of the original space, which limits the application of the method. Our approach can work for a weaker assumption, which does not require the inheritance of latent space. If the geometry of manifold latent space is available, it becomes possible to estimate the latent space and conduct optimizations within that space. Nevertheless, using manifold approximation techniques, as demonstrated in the works of \cite{MrGap,Sober2019ManifoldAB,Niyogi2008FindingTH},  requires a large amount of data as well as some strict conditions about its density to achieve accuracy. This contradicts the data-efficient perspective of BO.

\section{Conclusion}
In this paper, we proposed an efficient method for HD-BO with the assumption that the objective function $f$ has an effective low-dimensional manifold subspace. Our method with theoretical support uses a random linear projection combined with a surrogate model representation in the latent space, achieving the optimal posterior contraction rate. Additionally, our method optimizes BO’s acquisition function efficiently in the low-dimensional space while maintaining the convergence rate of the surrogate model. In contrast with other HD-BO methods, our approach can either efficiently exploit the geometry of the manifold whenever it is available or mitigate the overfitting phenomenon via semi-supervised learning. Empirically, our algorithm outperformed other high-dimensional BO baselines.

\authorrunning{Nguyen et al.}
% If there are two authors, write 'A.L. Benjamin and C.C. Broadus Jr.'
% If there are more than two authors, '[...] et al.' is used.

%\institute{FPT Software AI Center, Vietnam \email{nhquocanh@gmail.com}
%\and
%HaNoi University of Science and Technology, HUST, Vietnam \email{hungtt@soict.hust.edu.vn}}

%\tocauthor{Quoc-Anh~Nguyen,Hung~Tran}

%\toctitle{High-Dimensional Bayesian Optimization via Random Projection of Manifold Subspaces}

%\maketitle              % typeset the header of the contribution

\appendix 
\onecolumn 
\centerline{ \textbf{\huge Supplementary Material}}

\section{Proofs of Theorem 1} \label{appendix:proof_theo2}

\begin{lemma} \label{lemma:rank}
    Let $\mathbf{O} \in \mathbb{R}^{n \times k}$ is the orthogonal matrix (i.e $\mathbf{O}^T\mathbf{O}=\mathbb{I}_{k \times k}$). Let $\mathbf{A} \in \mathbb{R}^{n \times s}$ be the random orthogonal matrix. Suppose $k \leq s < n$. Then with probability 1, $\operatorname{rank}(\mathbf{O}^T\mathbf{A}) = k$.
\end{lemma}
\begin{proof}
    Since $\mathbf{A} \in \mathbb{R}^{n \times s}$ is the random orthogonal matrix, $\mathbf{A}$ uniformly distributed on the Stiefel manifold $\mathbb{V}_s(\mathbb{R}^n)$. Therefore, followed by Theorem 2.2.1 in \cite{Chikuse2003StatisticsOS}, $\mathbf{A}$ can be expressed as:
    \begin{equation}
        \mathbf{A} = Z(Z^TZ)^{-1/2}
    \end{equation}
    where $Z \in \mathbb{R}^{n \times s}$ is the random matrix where elements of $Z$ are i.i.d from normal distribution $\mathcal{N}(0,1)$.\\
    We have:
    \begin{align*}
        \mathbf{O}^T\mathbf{A} = \mathbf{O}^TZ(Z^TZ)^{-1/2}
    \end{align*}
    Since $(Z^TZ)^{-1/2} \in \mathbb{R}^{m \times m}$ is invertible, $\operatorname{rank}(\mathbf{O}^TZ(Z^TZ)^{-1/2}) = \operatorname{rank}(\mathbf{O}^TZ)$\\
    Moreover, followed by Lemma 2 in \cite{Wang2016BayesianOI}, with probability 1, $\operatorname{rank}(\mathbf{O}^TZ) = k.$
\end{proof}

\textbf{Given Lemma \ref{lemma:rank}, we now give a full proof of Theorem 1 in the main paper.}

\begin{proof}
    We have $\mathcal{M}$ is the $d-$dimensional manifold embedded in $\mathbb{R}^D$. According to the assumption, $\mathcal{M}$ can be embedded in Euclidean space $\mathbb{R}^m$. Therefore, $\exists$ a $m-$dimensional affine space $\mathcal{L} \subset \mathbb{R}^D$ such that $\mathcal{M} \subset \mathcal{L}$. Let the orthogonal basis of $\mathcal{L}$ is $\{b_i\}_{i=1}^m$. Defined $\Sigma = \begin{bmatrix}
 b_{1} &  b_{2} & \cdots & b_m \\
 \end{bmatrix} \in \mathbb{R}^{D \times m}$. Moreover, let $\alpha \in \mathbb{R}^D$ is the origin of $\mathcal{L}$.\\
    Let arbitrary point $x^* \in \mathcal{M} \subset \mathcal{L} \subset \mathbb{R}^D$. 
 Since $x^* \in \mathcal{L}$, there exists $q \in \mathbb{R}^m: x^* = \Sigma q + \alpha.$ With probability 1, $\Sigma^T\mathbf{A}^T$ has rank $m$  according to Lemma \ref{lemma:rank}. Therefore, with probability 1, $\exists z \in \mathbb{R}^m: \Sigma^T\mathbf{A}^Tz = q + \Sigma^T\alpha.$ Project $\mathbf{A}^Tz$ to $\mathcal{L}$ we have:
 \begin{align*}
     P_{\mathcal{L}}(\mathbf{A}^Tz) &= \Sigma\Sigma^T(\mathbf{A}^Tz-\alpha) + \alpha\\
     &=\Sigma\Sigma^T\mathbf{A}^Tz - \Sigma\Sigma^T\alpha + \alpha \\
     &= \Sigma (q + \Sigma^T\alpha) - \Sigma\Sigma^T\alpha + \alpha \\
     &= \Sigma q + \Sigma\Sigma^T\alpha - \Sigma\Sigma^T\alpha + \alpha \\
     &= \Sigma q + \alpha \\
     & = x^*
 \end{align*}
Denote $\mathcal{S}_{\mathcal{M}}(\mathbf{A}^Tz) = \{s \in \mathbb{R}^D\mid s \in \operatorname{argmin}_{m \in \mathcal{M}}||m-\mathbf{A}^Tz||\}$. Let arbitrary point $x_{\mathcal{M}} \in \mathcal{S}_{\mathcal{M}}(\mathbf{A}^Tz) \subset \mathcal{M} \subset \mathcal{L}$. 

We have $(\mathbf{A}^Tz - x^*) \bot (x^* - x_{\mathcal{L}})$ for all $x_{\mathcal{L}} \in \mathcal{L}$ since $x^*$ is the orthogonal projection of $\mathbf{A}^Tz$ to affine space $\mathcal{L}$. Therefore, $(\mathbf{A}^Tz - x^*) \bot (x^* - x_{\mathcal{M}})$. Consequently, we have:
 \begin{align*}
     \|\mathbf{A}^Tz - x_{\mathcal{M}}\|^2 &= \|\mathbf{A}^Tz - x^*\|^2 + \|x_{\mathcal{M}} - x^*\|^2\\
     & \geq \|\mathbf{A}^Tz - x^*\|^2 
 \end{align*}
 
$\Rightarrow x^* \in \mathcal{S}_{\mathcal{M}}(\mathbf{A}^Tz)$ since $x^* \in \mathcal{M}$.
 
  Next, we will prove that there is a unique point in the set $\mathcal{S}_{\mathcal{M}}(\mathbf{A}^Tz)$. Indeed, suppose there exists $x' \in \mathcal{S}_{\mathcal{M}}(\mathbf{A}^Tz)$ and $x' \neq x^*$. Since $x' \in \mathcal{L}, (\mathbf{A}^Tz - x^*) \bot (x^* - x')$. Therefore:
 \begin{align*}
     & ||\mathbf{A}^Tz - x'||^2 = ||\mathbf{A}^Tz-x^*||^2 + ||x^*-x'||^2 \\
     &\Rightarrow ||x^* - x'|| = 0 \text{ since } x^*, x' \in \mathcal{S}_{\mathcal{M}}(\mathbf{A}^Tz).
 \end{align*}
 $\Rightarrow$ This contradicts with the assumption that $x' \neq x^*$.
\end{proof}

\section{Discussion about diffeomorphism and Random projection} \label{appendix:diffeomorphism}
\subsection{Discussion on the diffeomorphism and the random projection} \label{appendix:diffeomorphism_small}
In this subsection, we will discuss further the diffeomorphism as we mentioned in Section 3.2 in the main paper, and prove two important results (1) why the random projection using a random matrix $A$ is a diffeomorphism onto its image; and (2) the condition that the projection $\Phi$ on the manifold is a diffeomorphism is enough to guarantee the convergence for the regression of the function $g$. 

\paragraph{The random projection using a random matrix $A$ is a diffeomorphism onto its image.}
Firstly, we state the definition of the diffeomorphism as follows.
\begin{definition}
Given two manifolds $\mathcal{M}$ and $\mathcal{N}$, a differentiable map $f: \mathcal{M} \rightarrow \mathcal{N}$ is called a \textbf{diffeomorphism} if it is a bijection and its inverse $f^{-1}: \mathcal{N} \rightarrow \mathcal{M}$ is differentiable as well. If these functions are $r$ times continuously differentiable, $f$ is called a \textbf{$C^r$-diffeomorphism}.

Two manifolds $\mathcal{M}$ and $\mathcal{N}$ are \textbf{diffeomorphic} if there is a diffeomorphism $f$ from $\mathcal{M}$ to $\mathcal{N}$. They are $C^r$-diffeomorphic if there is an $r$ times continuously differentiable bijective map between them whose inverse is also $r$ times continuously differentiable.
\end{definition}

Secondly, using a random matrix, we achieve the following important result: 
\begin{theorem} \label{theo:preserve}
\cite{Baraniuk2009RandomPO}
Let $\mathcal{M}$ be a compact $d-$dimensional manifold in $\mathbb{R}^D$ having volume $V$ and condition number $1/\tau$. Fix $0<\epsilon<1$ and $0<\rho<1$. Let $\mathbf{A}$ be a random orthogonal projection from $\mathbb{R}^D$ to $\mathbb{R}^m$ and 
\begin{equation} \label{condition}
    m = \mathcal{O}\left(\frac{dlog(DV\tau^{-1})log(\rho^{-1})}{\epsilon^2}\right)
\end{equation}
Suppose $m<D$. Then, with probability exceeding $1-\rho$, the following statement holds: For every points $x,y\in \mathcal{M}$:
\begin{equation} \label{eqa:preserve_appendix}
    (1-\epsilon)\sqrt{\frac{m}{D}} \leq \frac{d_2(\mathbf{A}x,\mathbf{A}y)}{d_2(x,y)} \leq (1+\epsilon)\sqrt{\frac{m}{D}}
\end{equation}
where $d_2(.,.)$ stands for Euclidean distance.
\end{theorem}
Theorem \ref{theo:preserve} resembles the Johnson-Lindenstrauss Lemma \cite{Gupta1999AnEP} but it is generalized to the manifold space and especially,  instead of being determined by the number of data points, the number of required dimensions for a structure-preserving mapping should be determined by the characteristics of the manifold itself. Moreover, as stated by \cite{JMLR:v17:14-230}, Equation (\ref{eqa:preserve_appendix}) implies that with high probability, the orthogonal projection matrix $\mathbf{A}$ is a diffeomorphism onto its image. Therefore, we can construct the diffeomorphism dimensionality reduction map $\Phi$ by the simple random orthogonal projection matrix $\mathbf{A}$ with suitable projected dimension $m$.

\paragraph{The regression convergence of function $g$.}
Learning directly the high-dimensional function $g: \mathcal{M} \rightarrow \mathbb{R}$ is expensive. Instead, as suggested by \cite{10.1214/15-AOS1390}, we will learn the low-dimensional GP $g_0$ with the diffeomorphism dimensionality reduction method. We construct a Gaussian Process in low-dimensional space $g_0 \sim \mathcal{GP}(\mu, k_m)$ with mean function $\mu: \Phi(\mathcal{M}) \subset \mathbb{R}^m \rightarrow \mathbb{R}$ and a positive-definite covariance function $k_a: \Phi(\mathcal{M}) \times \Phi(\mathcal{M}) \rightarrow \mathbb{R}$ defined as $k_a(x,y) = \exp\{-a^2 ||x-y||^2\}$. In addition, we specify the prior of the lengthscale parameter $a$ as $a^d \sim Ga(a_0,b_0)$ where $Ga(.,.)$ is the gamma distribution. We set the training data for Gaussian Process $g_0$ as $\{\Phi(x_i), y_i = g(x_i)+\epsilon_i\}_{i=1}^n$. The convergence of $g_0$ can be guaranteed in the following theorem.

\begin{theorem} \label{theo:contraction}
\cite{10.1214/15-AOS1390}
Assumes that $\mathcal{M}$ is a $d-$dimensional compact $C^{\gamma_1}$ submanifold of $\mathbb{R}^D$. Suppose that $\Phi: \mathbb{R}^D \rightarrow \mathbb{R}^m$ is a dimenionality reduction map such that $\Phi$ restricted on $\mathcal{M}$ is a $C^{\gamma_2}$-diffeomorphism. Then, for any $g\in C^s(\mathcal{M})$ with $s \leq \min\{2, \gamma_1 -1, \gamma_2 - 1\}$, the posterior contraction rate w.r.t $||.||_n$ of $g_0$ will be at least $\epsilon_n = n^{-s/(2s+d)}(logn)^{d+1}$.
\end{theorem}
We say the posterior contraction rate w.r.t $||.||_n$ of $g_0$ is at least $\epsilon_n$ if
\begin{equation} \label{eqa:posterior_contraction}
    \prod(||g-g_0||_n > \epsilon_n|S_n) \rightarrow 0, \text{in probability as }n \rightarrow \infty
\end{equation}
where $S_n = \{(x_i,y_i)\}_{i=1}^n$ denotes the dataset, $\prod(\mathcal{E}|S_n)$ is the posterior of $\mathcal{E}$ and $||g-g_0||_n^2 = \frac{1}{n}\sum_{i=1}^n (g(x_i)-g_0(\Phi(x_i)))^2$. 

Theorem \ref{theo:contraction} states that the discrepancy between $g_0$ and $g$ has a convergence rate of $n^{-s/(2s+d)}$ and it is shown that this is the optimal rate that can be obtained using Gaussian Process regression \cite{10.1214/15-AOS1390,JMLR:v17:14-230}. This is convenient for high-dimensional BO since we not only achieve the optimal posterior contraction rate for the surrogate model but also enable us to optimize the acquisition function in the low-dimensional projected space. 

\subsection{Additional Discussion on the advantage of a random projection}
Another advantage of a random orthogonal projection matrix is that it can reduce the noise effect of near-manifold data points as discussed in Section \ref{sect:discuss} in the main paper. Specifically, when $h$ is not accurately learned, the suggested points by $h$ may not lie on the manifold $\mathcal M$. For example, $h$ suggests the point $x_i$ while the true point in $\mathcal M$ should be $z_i$. We call the gap $\delta^i = x_i - z_i$ the noise. The more accurately the mapping $h$ is learned, the smaller the noise is. As shown in \cite{JMLR:v17:14-230}, the random projection $\mathbf{A}$ can reduce the negative effects of noise on the convergence of learning the function $g_0$. More precisely, let $A_{ij}$ be the element in $i-th$ row and $j-th$ column of random projection $\mathbf{A}$ and  $\delta^{i}_j$ be the $j-th$ element of noise $\delta^i$. Assume $\delta^i_j$ is i.i.d with $E[\delta^i_j] =0$ and $E[||\delta^i_j||^2] < \infty$ $\forall i = \overline{1,n}$. Then using the Lemma 2.9.5 in \cite{Vaart1996WeakCA}, we obtain:
\begin{equation*}
    \sqrt{D}\sum_{j=1}^D A_{ij}\delta^s_j \rightarrow \mathcal{N}(0, \text{cov}(\delta^s_1)) \quad \forall i=\overline{1,m}; \forall s = \overline{1,n}
\end{equation*}

Therefore, $\mathbf{A}\delta^s = \mathcal{O}(D^{-1/2}) \forall s = \overline{1,n}$, which reduces the magnitude of noise in the original space as high-dimensional $D$ increases. We remark that existing works using the encoder-decoder framework do not discuss the effect of these noises when learning feature mapping. In contrast, our proposed algorithm uses random projection, and therefore it can compress the noise of the near-manifold data points.

\subsection{Discussion on the convergence of the proposed RPM-BO algorithm} \label{appendix: converge_RPM_BO}
As we presented our approach for the high-dimensional Bayesian optimization problem in the main paper, instead of learning directly the objective function $f$, we propose to learn a feature mapping $h$ onto a high-dimensional latent space in a semi-supervised manner and then to learn the high-dimensional representation $g$ of the function $f$. 

By using random projection matrix $\mathbf{A}$ as we discussed in the main paper, we can learn exactly the function $g$ via a low-dimensional GP regression. Therefore, if we assume that the feature mapping $h$ is also accurately learned, then the convergence rate of our proposed BO algorithm is equivalent to the convergence rate of the a BO algorithm in low dimension $m$. For example, in the noise-free case, our proposed algorithm using EI converges at rate $\mathcal O(n^{-1/m})$, where $n$ is the number of samples and $m$ is the input dimension. The proof can be done similarly as in \cite{Wang2016BayesianOI}. In the noisy case, our algorithm converges at rate $\mathcal O(\frac{\gamma_n}{\sqrt{n}})$, where $\gamma_n$ is the maximum information gain of $n$ samples \cite{tran-the22a}.

%However, the requirement is that all data points must lie on manifold $\mathcal{M}$ which is intractable during the optimization of BO algorithm. Taking the advantage of mGP \cite{inproceedings}, we control the suggested points lying near the manifold by approximating the feature mapping $h$ by a parameterized function $h_{\theta}$, where $\theta$ is the learned parameters. By maximizing the marginal likelihood in equation \ref{eqa:mll}, we supervisely learn the true feature mapping $h$ through $h_{\theta}$. As a result, when $h_{\theta}$ is accurately learned, we can apply $h_{\theta}$ to the historical data points $\mathcal{X} = \{x_i\}_{i=1}^n$, and ensure that $h_{\theta}(\mathcal{X})$ lie on manifold $\mathcal{M}$. On the other hand, it requires a large amount of data to accurately learn $h$. Therefore, for most of iteration in BO algorithm, $h_{\theta}$ is just approximate the true feature mapping $h$ , and consequently $h_{\theta}(\mathcal{X})$ lie near the manifold $\mathcal{M}$. Fortunately, even data points $h_{\theta}(\mathcal{X})$ is not exactly lie on $\mathcal{M}$, as stated in \cite{JMLR:v17:14-230}, the random projection $\mathbf{A}$ reduces the negative effects of noise on the convergence of low-dimensioanl GP $g_0$. This is another advantage of random projection compared to the other dimensionality reduction techniques, which can compress the noise of the near-manifold data points. We will discuss further about the deleterious effect of noise and the convergence problem of our proposed method in the appendix.

\section{Experiment Details} \label{appendix:exp_setting}
\subsection{Detail setting for baselines}
All the algorithms use identical initial datasets in every replication to initialize the surrogate function, obtained via random Latin hypercube designs. For the VAE-BO model, we built an encoder having 4 hidden layers. The first layer has 512 hidden units followed by a ReLu activation function. The second layer has 256 hidden units followed by a tanh activation function. The third and the fourth layers have 64 and 32 hidden units, respectively and both apply the ReLu activation function. The decoder is the inverse of the encoder. In addition, we clipped the output of the decoder to the original search space $\mathcal{X}$ for the evaluation. Similar to \cite{Notin2021ImprovingBO}, we jointly train VAE model and the surrogate model. For computation efficiency, we only update the full model after 50 iterations. Since the dimension in latent is small, we use original UCB-BO instead of TuRBO as proposed in \cite{Notin2021ImprovingBO}. To train this VAE model, we created a synthetic training dataset including 5000 points sample uniformly in the search space $\mathcal{X}$. We used a GP model with the 5/2-Matern kernel and constant mean function for VAE-BO. We use Python 3, GPyTorch \cite{NEURIPS2018_27e8e171} and BoTorch \cite{balandat2020botorch}. For SAASBO, to reduce runtime and complexity, we reduce the $N_{\text{warmup}}=128$, $N_{\text{post}}=128$ and $L=8$ as suggested by the author in \cite{Eriksson2021HighDimensionalBO}. For TuRBO, we use 5 trust regions and set the batch size to 10 as suggested by the author. For all other algorithms, we used the authors’reference implementations. Especially, for the SILBO model, we use the \textit{bottom up} approach. The code for REMBO, SIRBO, HeSBO, and SILBO can be found at \url{https://github.com/cjfcsjt/SILBO}, for SAASBO can be bound at \url{https://github.com/martinjankowiak/saasbo} and TuRBO can be found at \url{https://github.com/uber-research/TuRBO}.

\subsection{Detail setting for RPM-BO}
For our method RPM-BO, we chose $a_0 = 1, b_0 = 0.15$ so that the lengthscale parameter $a$ of low-dimensional GP $g_0$ follows $Ga(1,0.15)$. We also chose SE (squared exponential) kernel as a kernel for the mGP $f$: $k_m(x_i,x_j) = k(h(x_i), h(x_j)) = \exp\{-a^2 ||h(x_i)-h(x_j)||^2\}$. To build the semi-supervised loss function, we set $\gamma =1, p = 5, q = 100$. To learn the parameter $\theta_h$ of RPM-BO, we maximized the semi-supervised loss function of a mGP $f$ with Adam optimizer. In case we cannot mathematically represent the formulation of mapping $h$, we will construct $h$ as a neural network from $\mathbb{R}^D \rightarrow \mathbb{R}^D$ with single 35 units hidden layer, and as the activation function, we use ReLU activation. To re-scale the output into the search space $\mathcal{X}$, we divide all the elements of the NN's output by the maximum absolute value of the elements: $v(x) = \frac{x}{\operatorname{max}_{i=\overline{1,D}}|x_i|}$ where $x = (x_1, x_2, ..., x_D)$. To optimize the acquisition function, we use the L-BFGS-B method.

%$$
%f_{\text{Levy}}(\mathbf{x}) = \sin^2(\pi w_1) + \sum_{i=1}^{d-1} (w_i-1)^2 [1 + 10\sin^2(\pi w_i + 1)] + (w_d - 1 )^2 [1+ \sin^2(2\pi w_d)], \\
%$$
%\text{where} $w_i = 1 + \frac{x_i - 1}{4} \forall i = \overline{1,d}$
% \begin{align*}
% f_{\text{Levy}}(\mathbf{x}) &= \sin^2(\pi w_1) + \sum_{i=1}^{d-1} (w_i-1)^2 [1 + 10\sin^2(\pi w_i + 1)] + (w_d - 1 )^2 [1+ \sin^2(2\pi w_d)], &&
% \end{align*}
% where $w_i = 1 + \frac{x_i - 1}{4} \forall i = \overline{1,d}$
% % $$ 
% % f_{\text{Rotated Hyper-Ellipsoid}}(\mathbf{x}) = \sum_{i=1}^d \sum_{j=1}^i x_j^2
% % $$
% \begin{align*}
% f_{\text{Rotated Hyper-Ellipsoid}}(\mathbf{x}) = \sum_{i=1}^d \sum_{j=1}^i x_j^2 &&
% \end{align*}

\section{Supplementary Results} \label{appendix:exp_result}
\subsection{Statistical significance test}
We show the p-values to test the performance of RPM-BO for the experiments in the main paper. Indeed, we use Wilcoxon Signed-Rank Test with the real value samples $\{X_i - Y_i\}_{i=1}^{20}$ where $\{X_i\}$ are the best function value of the proposed method (lower is better) and $\{Y_i\}$ are the best function value of baseline methods. The Alternative Hypothesis is the one-sided hypothesis where the distribution underlying $X_i-Y_i$ is stochastically less than a distribution symmetric about zero. If the $p$-value is less than 0.05, we can reject the Null Hypothesis, which says the median distribution of the difference is located at 0. The result is shown in Table \ref{tab: p_value_synthetic_AS},\ref{tab: p_value_synthetic_HS}, \ref{tab: p_value_synthetic_AM},  \ref{tab: p_value_synthetic_HM}, \ref{tab: p_value_real}. Overall, the p-value is less than 0.05 for most cases leading to the fact that we can reject the Null Hypothesis.

\begin{table*}[t]
            \caption{p-values between RPM-BO and baselines on Ackley Sphere experiments.}
            \centering
            \begin{tabular}{c|c|c|c}
            \toprule
            \textit{Ackley Sphere}  & \textbf{500D-10d} & \textbf{1000D-10d} & \textbf{1500D-10d} \\
            %\midrule
            \textbf{HeSBO}  & $1.30\text{E}-4$ & $3.14\text{E}-5$ & $6.80\text{E}-3$ \\
            \textbf{SILBO}  & $9.53\text{E}-7$ & $1.81\text{E}-5$ & $1.97\text{E}-4$ \\
            \textbf{SIRBO}  & $9.53\text{E}-6$ & $2.86\text{E}-6$ & $1.30\text{E}-4$ \\
            \textbf{REMBO}  & $9.53\text{E}-7$ & $9.53\text{E}-7$ & $9.53\text{E}-7$ \\
            \textbf{VAE-BO}  & $9.53\text{E}-7$ & $2.86\text{E}-6$ & $9.53\text{E}-7$ \\
            \textbf{TuRBO}  & $6.07\text{E}-1$ & $6.67\text{E}-6$ & $6.67\text{E}-6$ \\
            \textbf{Random}  & $9.53\text{E}-7$ & $1.91\text{E}-6$ & $9.53\text{E}-7$ \\
            \textbf{SAASBO}  & $2.86\text{E}-2$ & $6.83\text{E}-3$ & $9.76\text{E}-4$ \\
            \hline
            \end{tabular}
            
            \label{tab: p_value_synthetic_AS}
        \end{table*}

\begin{table*}[t]
            \caption{p-values between RPM-BO and baselines on Hyper Sphere experiments.}
            \centering
            \begin{tabular}{c|c|c|c}
            \toprule
            \textit{Hyper Sphere}  & \textbf{500D-10d} & \textbf{1000D-10d} & \textbf{1500D-10d} \\
            %\midrule
            \textbf{HeSBO}  & $3.14\text{E}-5$ & $1.90\text{E}-6$ & $1.90\text{E}-6$ \\
            \textbf{SILBO}  & $1.22\text{E}-4$ & $2.86\text{E}-6$ & $1.90\text{E}-6$ \\
            \textbf{SIRBO}  & $3.14\text{E}-5$ & $4.25\text{E}-4$ & $9.53\text{E}-7$ \\
            \textbf{REMBO}  & $9.53\text{E}-7$ & $9.53\text{E}-7$ & $9.53\text{E}-7$ \\
            \textbf{VAE-BO}  & $2.86\text{E}-6$ & $9.53\text{E}-7$ & $9.53\text{E}-7$ \\
            \textbf{TuRBO}  & $1.57\text{E}-3$ & $1.90\text{E}-6$ & $9.53\text{E}-7$ \\
            \textbf{Random}  & $1.90\text{E}-6$ & $9.53\text{E}-7$ & $9.53\text{E}-7$ \\
            \textbf{SAASBO}  & $1.22\text{E}-4$ & $1.22\text{E}-4$ & $3.05\text{E}-3$ \\
            \hline
            \end{tabular}
            
            \label{tab: p_value_synthetic_HS}
        \end{table*}

\begin{table*}[t]
            \caption{p-values between RPM-BO and baselines on Ackley Mix experiments.}
            \centering
            \begin{tabular}{c|c|c|c}
            \toprule
            \textit{Ackley Mix}  & \textbf{500D-10d} & \textbf{1000D-10d} & \textbf{1500D-10d} \\
            %\midrule
            \textbf{HeSBO}  & $9.53\text{E}-7$ & $9.53\text{E}-7$ & $9.53\text{E}-7$ \\
            \textbf{SILBO}  & $9.53\text{E}-7$ & $9.53\text{E}-7$ & $9.53\text{E}-7$ \\
            \textbf{SIRBO}  & $9.53\text{E}-7$ & $9.53\text{E}-7$ & $9.53\text{E}-7$ \\
            \textbf{REMBO}  & $9.53\text{E}-7$ & $9.53\text{E}-7$ & $9.53\text{E}-7$ \\
            \textbf{VAE-BO}  & $9.53\text{E}-7$ & $9.53\text{E}-7$ & $9.53\text{E}-7$ \\
            \textbf{TuRBO}  & $1.91\text{E}-6$ & $4.76\text{E}-6$ & $9.53\text{E}-7$ \\
            \textbf{Random}  & $9.53\text{E}-7$ & $9.53\text{E}-7$ & $9.53\text{E}-7$ \\
            \textbf{SAASBO}  & $1.22\text{E}-4$ & $9.53\text{E}-7$ & $9.53\text{E}-7$ \\
            \hline
            \end{tabular}
            
            \label{tab: p_value_synthetic_AM}
        \end{table*}

\begin{table*}[t]
            \caption{p-values between RPM-BO and baselines on Hyper Mix experiments.}
            \centering
            \begin{tabular}{c|c|c|c}
            \toprule
            \textit{Hyper Mix}  & \textbf{500D-10d} & \textbf{1000D-10d} & \textbf{1500D-10d} \\
            %\midrule
            \textbf{HeSBO}  & $2.86\text{E}-6$ & $9.53\text{E}-7$ & $9.53\text{E}-7$ \\
            \textbf{SILBO}  & $2.86\text{E}-6$ & $9.53\text{E}-7$ & $9.53\text{E}-7$ \\
            \textbf{SIRBO}  & $3.53\text{E}-4$ & $9.53\text{E}-7$ & $9.53\text{E}-7$ \\
            \textbf{REMBO}  & $9.53\text{E}-7$ & $9.53\text{E}-7$ & $9.53\text{E}-7$ \\
            \textbf{VAE-BO}  & $1.97\text{E}-4$ & $9.53\text{E}-7$ & $9.53\text{E}-7$ \\
            \textbf{TuRBO}  & $6.04\text{E}-4$ & $9.53\text{E}-7$ & $9.53\text{E}-7$ \\
            \textbf{Random}  & $6.67\text{E}-5$ & $9.53\text{E}-7$ & $9.53\text{E}-7$ \\
            \textbf{SAASBO}  & $1.22\text{E}-4$ & $9.53\text{E}-7$ & $9.53\text{E}-7$ \\
            \hline
            \end{tabular}
            
            \label{tab: p_value_synthetic_HM}
        \end{table*}

\begin{table*}[t]
            \caption{p-values between RPM-BO and baselines on Real-world experiments.}
            \centering
            \begin{tabular}{c|c|c|c}
            \toprule
            \textit{Real-world}  & \textbf{Lasso Leukemia} & \textbf{Lasso Hard} & \textbf{Mujoco Ant} \\
            %\midrule
            \textbf{HeSBO}  & $9.53\text{E}-7$ & $9.53\text{E}-7$ & $9.76\text{E}-4$ \\
            \textbf{SILBO}  & $4.76\text{E}-6$ & $9.53\text{E}-7$ & $2.92\text{E}-3$ \\
            \textbf{SIRBO}  & $6.67\text{E}-6$ & $9.53\text{E}-7$ & $9.76\text{E}-4$ \\
            \textbf{REMBO}  & $9.53\text{E}-7$ & $9.53\text{E}-7$ & $9.53\text{E}-7$ \\
            \textbf{VAE-BO}  & $1.90\text{E}-6$ & $3.14\text{E}-5$ & $1.11\text{E}-1$ \\
            \textbf{TuRBO}  & $1.61\text{E}-4$ & $1.04\text{E}-4$ & $9.76\text{E}-4$ \\
            \textbf{Random}  & $9.53\text{E}-7$ & $9.53\text{E}-7$ & $9.76\text{E}-4$ \\
            \textbf{SAASBO}  & $3.66\text{E}-4$ & $6.71\text{E}-3$ & $3.05\text{E}-3$ \\
            \hline
            \end{tabular}
            
            \label{tab: p_value_real}
        \end{table*}

\subsection{On the choice of projection dimension} \label{appendix:test_m}
We do experiments to test the performance of our method for different values of $m$. We re-use $D$-dimensional Ackley function with an effective $15$-dimensional subspace with a mixed structure. We use three different values of $m$: $m=10,15,20$. Figure \ref{fig:m} shows the result of our methods for different values of $m$. We can see  that the performance RPM-BO does not deviate too much for different values of $m$. The value of $m$ also represents the dimensionality of the search space of the acquisition function. Therefore, it directly affects the computing burden of maximizing acquisition function, and the effectiveness of BO optimization. In our experiments on $m$, we can observe that our algorithm achieves robustness regarding the choice of projected dimension.

\begin{figure}
    \centering
    \includegraphics[width=1.0\linewidth]{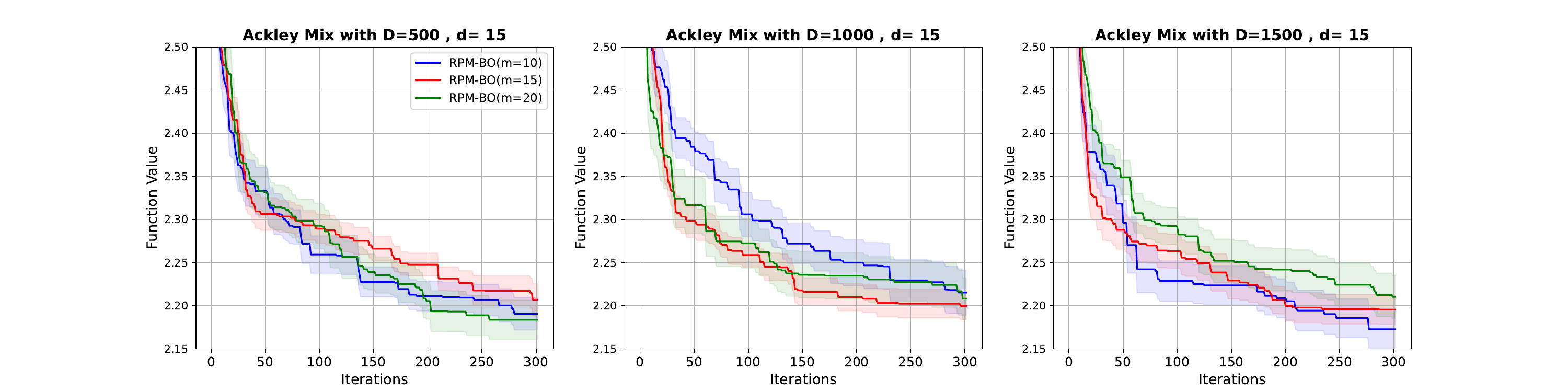}
    \caption{Performance on Ackley function with effective 15-dimensional mix manifold $\mathbf M$ for varying projection dimension $m$.}
    \label{fig:m}
\end{figure}

\subsection{Runtime experiment} \label{appendix:runtime}
We measure the runtime of RPM-BO and each baseline method on the Ackley Mix 1000D test problem. We run each method for 300 evaluations independently using 20 random seeds and calculate the average wall clock time per restart. Note that for SAASBO, we only run 100 evaluations per seed due to its high complexity. The result is shown in Table \ref{tab:runtime}. As we can see, although the run time of RPM-BO is the second longest, it is still 8 times faster than SAASBO. However, the result of RPM-BO is superior to the other baselines. As expected, the Random Search method achieves the fastest wall clock time and the random projection-based method have a similar run time to each other.

\begin{table*}[h]
    \caption{Average runtime per restart on the Ackley with effective mixed manifold for 1000 input dimensions. Runtimes are obtained using AMD Ryzen Threadripper Pro 3955WX 16-Cores CPU.} 
    \centering
    \begin{tabular}{lccccc}
\hline \textbf{Methods} & \textbf{Time/restart (seconds)} & \textbf{Best function value}\\
\hline $\operatorname{RPM-BO}$  & $1470.5439$ & $\mathbf{2.2070 \pm 0.1685}$\\
$\operatorname{HeSBO}$          & $80.8997$   & $3.1328 \pm 0.2223$\\
$\operatorname{SIRBO}$          & $57.6311$   & $3.1151 \pm 0.0755$\\
$\operatorname{SILBO}$          & $148.0933$  & $3.3073 \pm 0.1120$ \\
$\operatorname{REMBO}$          & $61.2815$   & $3.5734 \pm 0.2168$\\
$\operatorname{TuRBO}$          & $78.7863$   & $2.9815 \pm 0.4203$\\
$\operatorname{Random Search}$  & $\leq 5$    & $3.5123 \pm 0.0771$\\
$\operatorname{SAASBO}$         & $9358.7181$ & $2.8542 \pm 0.2543$\\
\hline
\end{tabular}
    
    \label{tab:runtime}
\end{table*}

\subsection{On the overfitting of the orthogonal projection $h$} \label{appendix:overfit}
We consider the overfitting issue of mapping $h$ when the training set is small. We create the toy experiment as follows.  We choose the Hyper function with effective spherical manifold $\mathcal{S}^{10}$ for the objective function. We randomly create 20 test points in $[-1,1]^{1000}$. Next, we randomly create 5 labeled datasets each with sizes of 10, 30, 50, 70, and 100. For the unlabeled data, we uniformly random 100 points $x_i'$ in $[-1,1]^{1000}$ and 10 coefficients $\lambda_j$ in $(0,1)$. We train the mapping $h$ in four settings. The first setting ($\operatorname{Original-GP}$) is the original GP without using $h$; the second setting ($\operatorname{Sup-NN}$) is using a neural net with the supervised loss function in Equation (6) in the main paper; the third setting ($\operatorname{Semi-NN}$) is training the neural net with the semi-supervised loss function in Equation (8) in the main paper and the last setting ($\operatorname{Geometry-aware}$) is using the geometry-aware spherical representation in Equation (5) in the main paper. After the training phase, we calculate the following criteria on the test set:
\begin{equation} \label{eqa:pos_loss}
    L(h)=\Sigma_{k=1}^{100}\frac{\Sigma_{(x,y) \in \text{test set}}\left(g^{(k)}(h(x)) - y\right)^2}{|\text{test set}|}
\end{equation}
where $g^{(k)}$ is the sample of the trained high-dimensional GP $g$ corresponding to $h$. A lower value in Equation (\ref{eqa:pos_loss}) implies a faster posterior contraction rate defined in Equation (\ref{eqa:posterior_contraction}), hence a lower overfitting issue on the test set. We run each experiment 30 times and report the mean, and standard deviation in Table \ref{tab: overfit}. As we can see, with semi-supervised training, our trained model achieves a lower error on the test set than supervised training across the different sizes of labeled datasets. Moreover, if we can exploit the geometry property into the representation of mapping $h$, we achieve a superior result compared to the semi-supervised training of NN. It is worth noting that using geometry-aware representation in Equation (8) in the main paper nullifies the unsupervised consistency loss, allowing for a purely supervised training approach without the necessity of unlabeled data. Furthermore, in contrast to the $\operatorname{Original-GP}$ approach, training GP through our composition results in lower loss values. This underscores the observation that training the GP on a manifold yields faster convergence.

\begin{table*}[h]
    \caption{Comparison of loss function defined in Equation (\ref{eqa:pos_loss}) for four settings.}
    \centering
    \begin{tabular}{lccccc}
\hline \textbf{Labels} & \textbf{10} & \textbf{30} & \textbf{50} & \textbf{70} & \textbf{100}\\
\hline $\operatorname{Original-GP}$  & $113.77 \pm 1.94$ & $127.83 \pm 3.21$ & $119.28 \pm 2.41$ & $117.50 \pm 2.86$ & $115.63 \pm 3.24$\\
$\operatorname{Sup-NN}$  & $215.61 \pm 27.53$ & $138.23 \pm 18.61$ & $106.29 \pm 7.34$ & $99.71 \pm 9.54$ & $95.54 \pm 7.69$\\
$\operatorname{Semi-NN}$        & $103.25 \pm 11.65$ & $100.77 \pm 14.67$ & $94.51 \pm 12.23$ & $92.22 \pm 11.64$ & $90.49 \pm 9.14$\\
$\operatorname{Geometry-aware}$ & $95.84 \pm 1.33$ & $93.37 \pm 1.42$ & $90.96 \pm 16.60$ & $84.42 \pm 8.83$ & $81.70 \pm 7.57$\\
\hline
\end{tabular}
    
    \label{tab: overfit}
\end{table*}

\subsection{Additional Synthetic Experiments} \label{appendix:additional_synthetic}
\paragraph{Linear $d$-dimensional manifold} We build function $f(x) = g(Rx)$ where $R \in \mathbb R^{d \times D}$ is the orthogonal matrix. This leads to the fact that $f$ has an effective dimension as stated in \cite{Wang2016BayesianOI}. Hence, we can create the feature map $h(x) = h_{B}(x) = BB^Tx$ where $B \in  \mathbb R^{D \times d}$ is an orthogonal matrix. We use Ackley and Rotated-Hyper Ellipsoid to build above the function $g$. We set $m=d=10$. The result is shown in Figure \ref{img:linear}. As we can see, RPM-BO still achieves good performance and outperform the other baselines in the Ackley function.
\begin{figure}[h]
    \centering
    \includegraphics[width=1.0\linewidth]{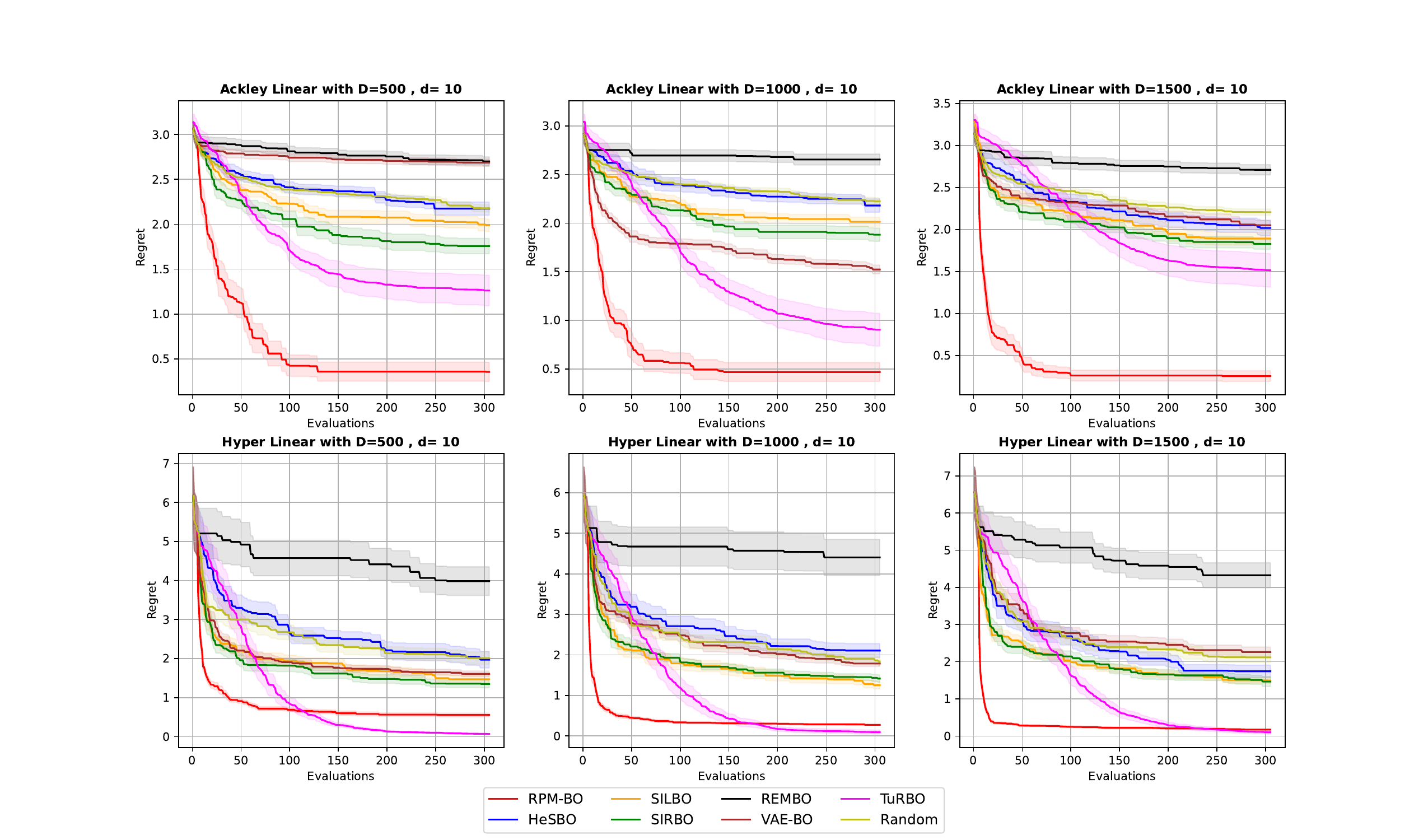}
    \caption{Performances on two standard functions with effective linear manifold for 500, 1000, and 1500 input dimensions. For all cases, the dimension of an effective manifold is 10. The y-axis presents the regret}
    \label{img:linear}
\end{figure}
\paragraph{Test with Levy function}
We test out the method with other synthetic functions. In this case, we test on Levy function. The result is shown in Figure \ref{img:levy_sphere}, \ref{img:levy_mix}. In the spherical case, RPM-BO outperforms the other baseline within 300 evaluations. We also see a decrease in the performance of the TuRBO method when the input dimension gets higher. SAASBO seems to run well at the first 60 evaluations, but due to the limited budget, it can not run with more evaluations. In the mixed case, SAASBO shows the best performance in all three input dimensions. TuRBO shows good performance with 500D Levy Mix experiments but gradually decreases the result. On the other hand, our method still runs competitively with the other baselines.

\paragraph{Test with larger number of iterations}
 To furtfer test the convergence of RPM-BO, we present results for four synthetic functions in Figure \ref{img:large_T}  with 2000 iterations and 5 independent runs. In general, RPM-BO exhibits faster convergence compared to other baseline methods with competitive performance. This underscores the ability to achieve a rapid convergence rate even when the original input dimension is high. Compared to the local trust-region method, TuRBO necessitates exploration in the high-dimensional space to reach the optimum, leading to a requirement for a larger number of iterations.

\begin{figure}[h]
    \centering
    \includegraphics[width=1.0\linewidth]{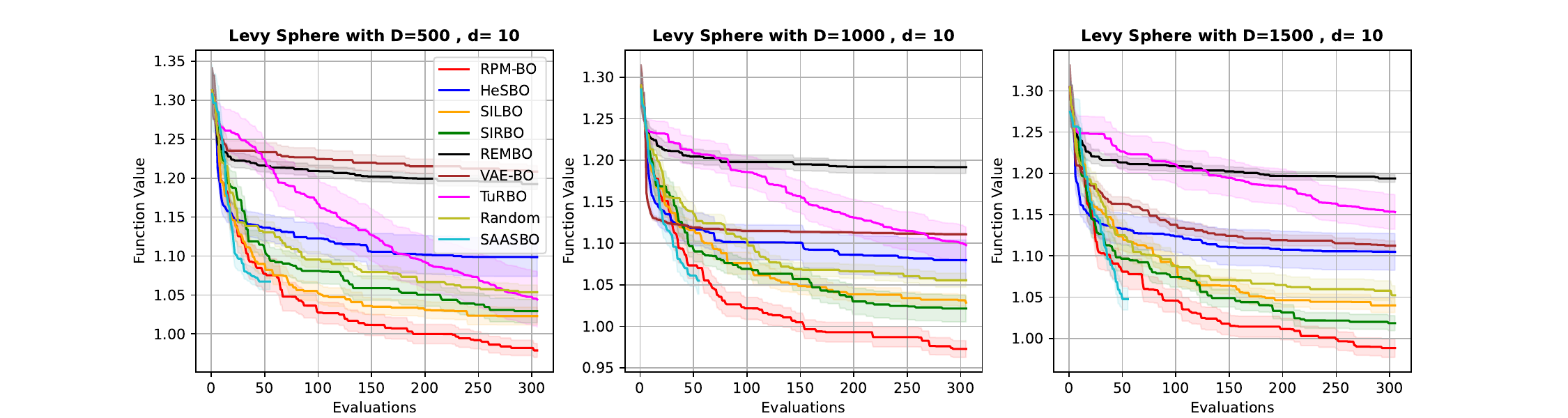}
    \caption{Performances on Levy function with effective sphere manifold for 500, 1000, 1500 input dimensions. For all cases, the dimension of an effective manifold is 10. The y-axis presents the function vale}
    \label{img:levy_sphere}
\end{figure}

\begin{figure}[h]
    \centering
    \includegraphics[width=1.0\linewidth]{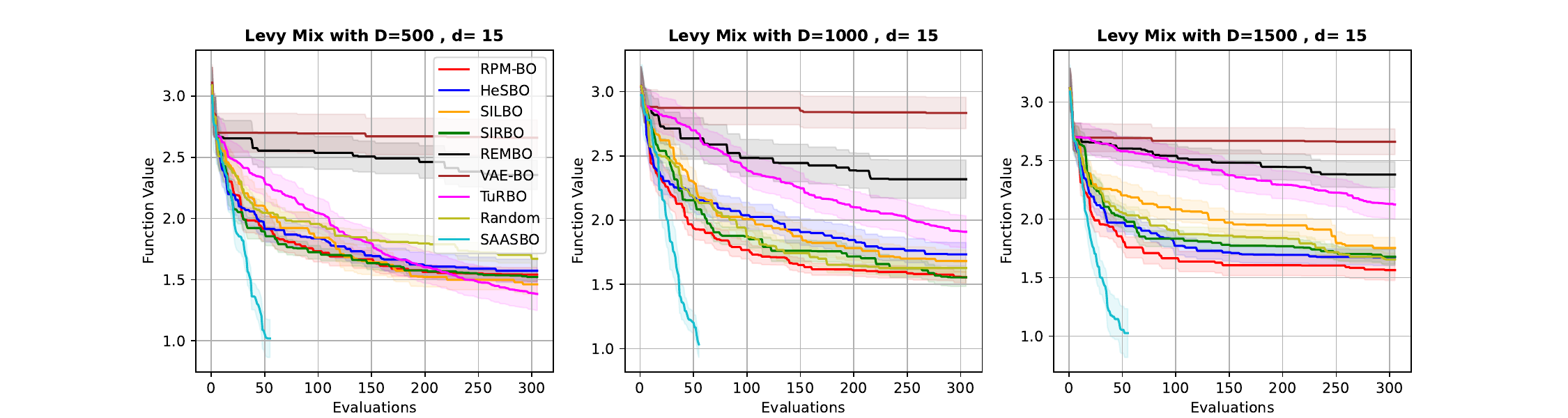}
    \caption{Performances on Levy function with effective mix manifold for 500, 1000, 1500 input dimensions. For all cases, the dimension of the effective manifold is 15. The y-axis presents the function value}
    \label{img:levy_mix}
\end{figure}

\begin{figure}[h]
    \centering
    \includegraphics[width=1.0\linewidth, height=.15\textwidth]{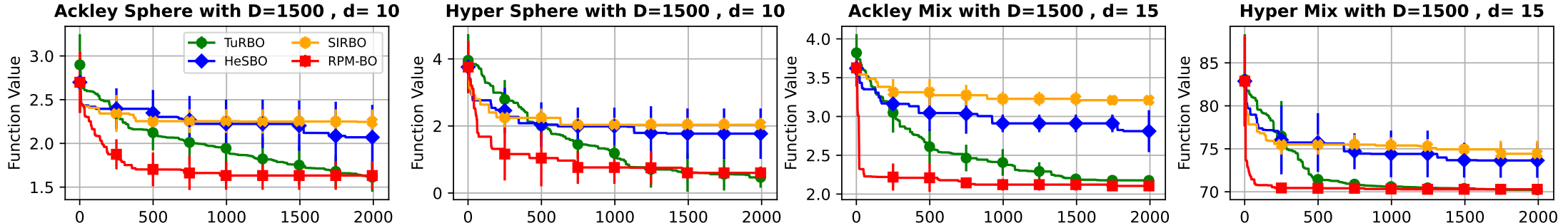}
    \caption{ Performance on 2 functions with 2 types of latent effective manifold.}
    \label{img:large_T}
\end{figure}

\subsection{Benchmark Test Functions}
This section gives the equations of the benchmark test functions considered in the experiments. Namely, we minimize the Ackley, Rotated Hyper-Ellipsoid functions defined as:
% $$
% f_{\text{Ackley}}(\mathbf{x}) = -20\exp\left({-0.2\sqrt{\frac{1}{d}\sum_{i=1}^d x_i^2}}\right) - \exp\left(\frac{1}{d}\sum_{i=1}^d\cos\{2\pix_i\}\right) + 20 +\exp(1)
% $$
\begin{align*}
&f_{\text{Ackley}}(\mathbf{x}) = -20\exp\left({-0.2\sqrt{\frac{\sum_{i=1}^d x_i^2}{d}}}\right)-\exp\left(\frac{\sum_{i=1}^d cos\{2\pi x_i\}}{d}\right) + 20 +\exp(1)\\
&f_{\text{Rotated Hyper-Ellipsoid}}(\mathbf{x}) = \sum_{i=1}^d \sum_{j=1}^i x_j^2\\
&f_{\text{Levy}}(\mathbf{x}) = sin^2(\pi w_1) + \sum_{i=1}^{d-1} (w_i-1)^2 [1 + 10 sin^2(\pi w_i + 1)] + (w_d - 1 )^2 [1+ sin^2(2\pi w_d)]
\end{align*}
where  $w_i = 1 + \frac{x_i - 1}{4} \forall i = \overline{1,d}$, 

\subsection{Real-world Experiments}
For more real-world applications, we apply the RPM-BO algorithm for MuJoCo Humanoid experiments \cite{6386109}. Humanoid is the most difficult task in MuJoCo. It has a total of 6392 parameters. In this experiment, we design the feature map $h$ as the neural network described in Appendix \ref{appendix:exp_setting}. We set the projection dimension for all methods as 15 and also $m=15$. The result is shown in Figure \ref{img:humanoid}. As we can see, our method shows the best performance followed by HeSBO, REMBO, and SILBO. Surprisingly, TuRBO shows the worst performance in this experiment. 

\begin{figure}
    \centering
    \includegraphics[width=0.6\linewidth]{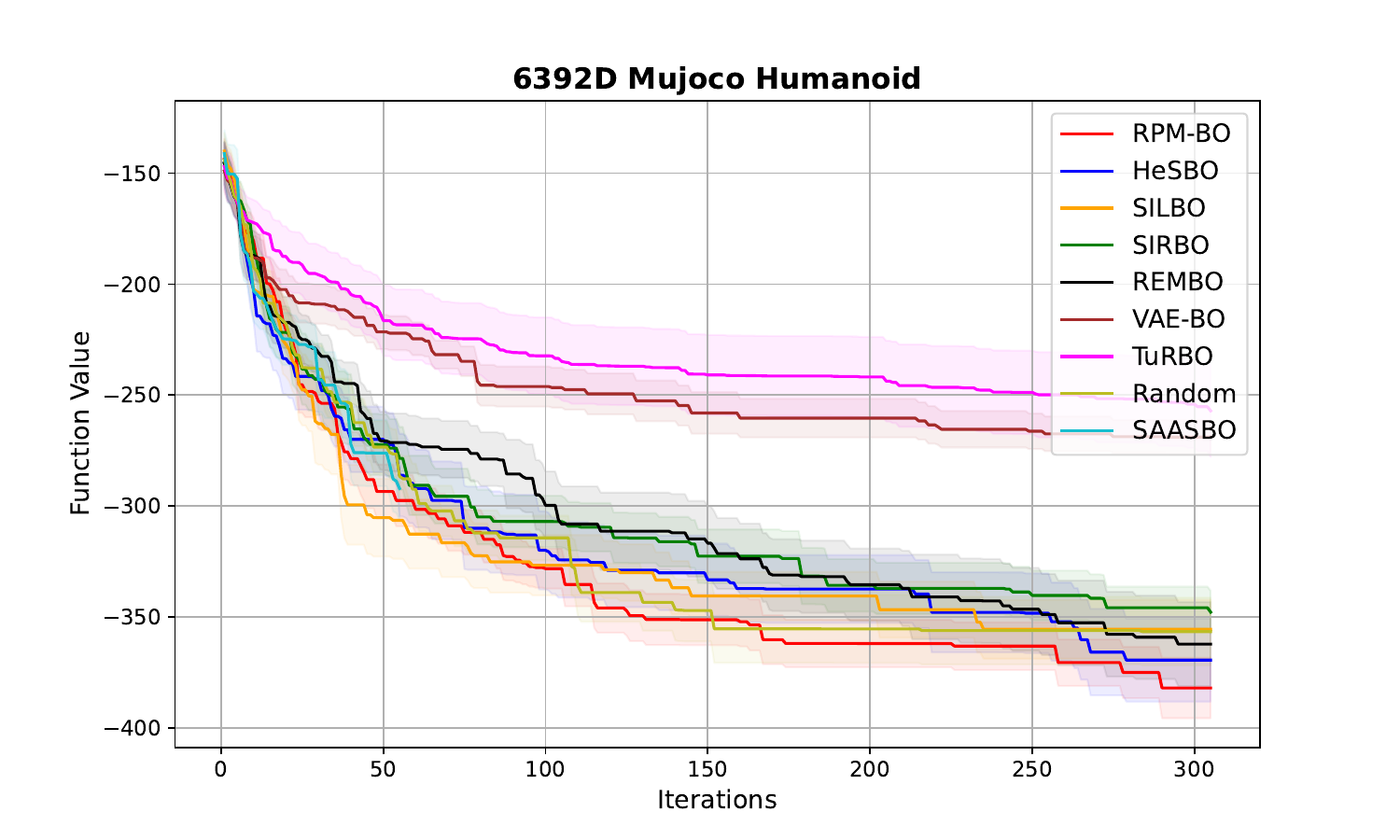}
    \caption{Performances on MuJoCo humanoid experiments. The y-axis show the negative reward (the lower is better).}
    \label{img:humanoid}
\end{figure}


\begin{thebibliography}{88}

\bibitem{balandat2020botorch}
Balandat, M., Karrer, B., Jiang, D.R., Daulton, S., Letham, B., Wilson, A.G., Bakshy, E.: {BoTorch: A Framework for Efficient Monte-Carlo Bayesian Optimization}. In: Advances in Neural Information Processing Systems 33 (2020)

\bibitem{Baraniuk2009RandomPO}
Baraniuk, R., Wakin, M.B.: Random projections of smooth manifolds. Foundations of Computational Mathematics  \textbf{9},  51--77 (2009)

\bibitem{JMLR:v13:bergstra12a}
Bergstra, J., Bengio, Y.: Random search for hyper-parameter optimization. Journal of Machine Learning Research  \textbf{13}(10),  281--305 (2012)

\bibitem{Berthelot2019}
Berthelot, D., Carlini, N., Goodfellow, I.J., Papernot, N., Oliver, A., Raffel, C.: Mixmatch: {A} holistic approach to semi-supervised learning. In: Advances in Neural Information Processing Systems 32, 2019. pp. 5050--5060 (2019)

\bibitem{Borovitskiy2020MaternGP}
Borovitskiy, V., Terenin, A., Mostowsky, P., Deisenroth, M.P.: Mat{\'{e}}rn gaussian processes on riemannian manifolds. In: Larochelle, H., Ranzato, M., Hadsell, R., Balcan, M., Lin, H. (eds.) Advances in Neural Information Processing Systems 33: Annual Conference on Neural Information Processing Systems 2020, NeurIPS 2020, December 6-12, 2020, virtual (2020)

\bibitem{inproceedings}
Calandra, R., Peters, J., Rasmussen, C., Deisenroth, M.: Manifold gaussian processes for regression (11 2016). \doi{10.1109/IJCNN.2016.7727626}

\bibitem{Calinon2019GaussiansOR}
Calinon, S.: Gaussians on riemannian manifolds: Applications for robot learning and adaptive control. IEEE Robotics \& Automation Magazine  \textbf{27},  33--45 (2019)

\bibitem{Carlsson2009TopologyAD}
Carlsson, G.E.: Topology and data. Bulletin of the American Mathematical Society  \textbf{46},  255--308 (2009)

\bibitem{chen2020semi}
Chen, J., Zhu, G., Yuan, C., Huang, Y.: Semi-supervised embedding learning for high-dimensional bayesian optimization. arXiv preprint arXiv:2005.14601  (2020)

\bibitem{Chikuse2003StatisticsOS}
Chikuse, Y.: Statistics on special manifolds (2003)

\bibitem{Dai2022}
Dai, Z., Shu, Y., Low, B.K.H., Jaillet, P.: Sample-then-optimize batch neural thompson sampling. In: NeurIPS (2022)

\bibitem{MrGap}
Dunson, D., Wu, N.: Inferring manifolds from noisy data using gaussian processes (10 2021)

\bibitem{Eriksson2021HighDimensionalBO}
Eriksson, D., Jankowiak, M.: High-dimensional bayesian optimization with sparse axis-aligned subspaces. In: Conference on Uncertainty in Artificial Intelligence (2021)

\bibitem{NEURIPS2019_6c990b7a}
Eriksson, D., Pearce, M., Gardner, J., Turner, R.D., Poloczek, M.: Scalable global optimization via local bayesian optimization. In: Advances in Neural Information Processing Systems. vol.~32 (2019)

\bibitem{NEURIPS2018_27e8e171}
Gardner, J., Pleiss, G., Weinberger, K.Q., Bindel, D., Wilson, A.G.: Gpytorch: Blackbox matrix-matrix gaussian process inference with gpu acceleration. In: Advances in Neural Information Processing Systems. vol.~31 (2018)

\bibitem{GmezBombarelli2018AutomaticCD}
G{\'o}mez-Bombarelli, R., Duvenaud, D.K., Hern{\'a}ndez-Lobato, J.M., Aguilera-Iparraguirre, J., Hirzel, T.D., Adams, R.P., Aspuru-Guzik, A.: Automatic chemical design using a data-driven continuous representation of molecules. ACS Central Science  \textbf{4},  268 -- 276 (2018)

\bibitem{JMLR:v17:14-230}
Guhaniyogi, R., Dunson, D.B.: Compressed gaussian process for manifold regression. Journal of Machine Learning Research  \textbf{17}(69),  1--26 (2016)

\bibitem{Gupta1999AnEP}
Gupta, A., Dasgupta, S.: An elementary proof of the johnson-lindenstrauss lemma (1999)

\bibitem{Hutchinson2021VectorvaluedGP}
Hutchinson, M., Terenin, A., Borovitskiy, V., Takao, S., Teh, Y.W., Deisenroth, M.P.: Vector-valued gaussian processes on riemannian manifolds via gauge independent projected kernels. In: NeurIPS (2021)

\bibitem{Jaquier2021GeometryawareBO}
Jaquier, N., Borovitskiy, V., Smolensky, A., Terenin, A., Asfour, T., Rozo, L.D.: Geometry-aware bayesian optimization in robotics using riemannian mat{\'e}rn kernels. In: CoRL (2021)

\bibitem{NEURIPS2020_f05da679}
Jaquier, N., Rozo, L.: High-dimensional bayesian optimization via nested riemannian manifolds. In: Advances in Neural Information Processing Systems. vol.~33, pp. 20939--20951 (2020)

\bibitem{Jaquier2019BayesianOM}
Jaquier, N., Rozo, L.D., Calinon, S., B{\"u}rger, M.: Bayesian optimization meets riemannian manifolds in robot learning. In: CoRL (2019)

\bibitem{Johnson1984ExtensionsOL}
Johnson, W.B.: Extensions of lipschitz mappings into hilbert space. Contemporary mathematics  \textbf{26},  189--206 (1984)

\bibitem{Kandasamy2015HighDB}
Kandasamy, K., Schneider, J.G., P{\'o}czos, B.: High dimensional bayesian optimisation and bandits via additive models. In: ICML (2015)

\bibitem{Kirschner2019AdaptiveAS}
Kirschner, J., Mutn{\'y}, M., Hiller, N., Ischebeck, R., Krause, A.: Adaptive and safe bayesian optimization in high dimensions via one-dimensional subspaces. In: ICML (2019)

\bibitem{Samuli2017}
Laine, S., Aila, T.: Temporal ensembling for semi-supervised learning. In: 5th International Conference on Learning Representations, {ICLR} 2017, Toulon, France, April 24-26, 2017, Conference Track Proceedings. OpenReview.net (2017)

\bibitem{Leobacher2021}
Leobacher, G., Steinicke, A.: Existence, uniqueness and regularity of the projection onto differentiable manifolds. Annals of Global Analysis and Geometry  \textbf{60},  1--29 (10 2021). \doi{10.1007/s10455-021-09788-z}

\bibitem{NEURIPS2020_10fb6cfa}
Letham, B., Calandra, R., Rai, A., Bakshy, E.: Re-examining linear embeddings for high-dimensional bayesian optimization. In: Advances in Neural Information Processing Systems. vol.~33, pp. 1546--1558 (2020)

\bibitem{Li2017HighDB}
Li, C., Gupta, S., Rana, S., Nguyen, V., Venkatesh, S., Shilton, A.: High dimensional bayesian optimization using dropout. In: IJCAI (2017)

\bibitem{Li2016HighDB}
Li, C.L., Kandasamy, K., P{\'o}czos, B., Schneider, J.G.: High dimensional bayesian optimization via restricted projection pursuit models. In: AISTATS (2016)

\bibitem{10.1007/3-540-07165-2_55}
Mo{\v{c}}kus, J.: On bayesian methods for seeking the extremum. In: Optimization Techniques IFIP Technical Conference Novosibirsk, July 1--7, 1974. pp. 400--404. Berlin, Heidelberg (1975)

\bibitem{moriconi2020high}
Moriconi, R., Deisenroth, M.P., Sesh~Kumar, K.: High-dimensional bayesian optimization using low-dimensional feature spaces. Machine Learning  \textbf{109}(9),  1925--1943 (2020)

\bibitem{Nayebi2019AFF}
Nayebi, A., Munteanu, A., Poloczek, M.: A framework for bayesian optimization in embedded subspaces. In: ICML (2019)

\bibitem{Niyogi2008FindingTH}
Niyogi, P., Smale, S., Weinberger, S.: Finding the homology of submanifolds with high confidence from random samples. Discrete \& Computational Geometry  \textbf{39},  419--441 (2008)

\bibitem{Notin2021ImprovingBO}
Notin, P., Hern{\'a}ndez-Lobato, J.M., Gal, Y.: Improving black-box optimization in vae latent space using decoder uncertainty. In: NeurIPS (2021)

\bibitem{45016f41f3844a1087845a2b542f5da9}
Papenmeier, L., Nardi, L., Poloczek, M.: Increasing the scope as you learn: Adaptive bayesian optimization in nested subspaces. In: Advances in Neural Information Processing Systems, NeurIPS 2022. vol.~35 (2022)

\bibitem{Persson2014TheWE}
Persson, M.: The whitney embedding theorem (2014)

\bibitem{8461237}
Rai, A., Antonova, R., Song, S., Martin, W., Geyer, H., Atkeson, C.: Bayesian optimization using domain knowledge on the atrias biped. In: 2018 IEEE International Conference on Robotics and Automation (ICRA). pp. 1771--1778 (2018). \doi{10.1109/ICRA.2018.8461237}

\bibitem{RASMUSSEN2005}
Rasmussen, C.E., Williams, C.K.I.: Gaussian Processes for Machine Learning (Adaptive Computation and Machine Learning). The MIT Press (2005)

\bibitem{pmlr-v84-rolland18a}
Rolland, P., Scarlett, J., Bogunovic, I., Cevher, V.: High-dimensional bayesian optimization via additive models with overlapping groups. In: Proceedings of the Twenty-First International Conference on Artificial Intelligence and Statistics (2018)

\bibitem{NIPS2012_05311655}
Snoek, J., Larochelle, H., Adams, R.P.: Practical bayesian optimization of machine learning algorithms. In: Advances in Neural Information Processing Systems. vol.~25 (2012)

\bibitem{Sober2019ManifoldAB}
Sober, B., Levin, D.: Manifold approximation by moving least-squares projection (mmls). Constructive Approximation  (2019)

\bibitem{6138914}
Srinivas, N., Krause, A., Kakade, S.M., Seeger, M.W.: Information-theoretic regret bounds for gaussian process optimization in the bandit setting. IEEE Transactions on Information Theory  \textbf{58}(5),  3250--3265 (2012). \doi{10.1109/TIT.2011.2182033}

\bibitem{Antti2017}
Tarvainen, A., Valpola, H.: Mean teachers are better role models: Weight-averaged consistency targets improve semi-supervised deep learning results. In: 5th International Conference on Learning Representations, 2017, Workshop Track Proceedings

\bibitem{Thomas2009}
Thomas, P.: Semi-supervised learning by olivier chapelle, bernhard schölkopf, and alexander zien (review). IEEE Transactions on Neural Networks  \textbf{20}, ~542 (01 2009)

\bibitem{6386109}
Todorov, E., Erez, T., Tassa, Y.: Mujoco: A physics engine for model-based control. In: 2012 IEEE/RSJ International Conference on Intelligent Robots and Systems. pp. 5026--5033 (2012). \doi{10.1109/IROS.2012.6386109}

\bibitem{TranThe2020TradingCR}
Tran-The, H., Gupta, S., Rana, S., Venkatesh, S.: Trading convergence rate with computational budget in high dimensional bayesian optimization. In: AAAI (2020)

\bibitem{tran-the22a}
Tran-The, H., Gupta, S., Rana, S., Venkatesh, S.: Regret bounds for expected improvement algorithms in gaussian process bandit optimization. In: Proceedings of The 25th International Conference on Artificial Intelligence and Statistics (2022)

\bibitem{Vaart1996WeakCA}
van~der Vaart, A., Wellner, J.A.: Weak convergence and empirical processes: With applications to statistics (1996)

\bibitem{Wang2017BatchedHB}
Wang, Z., Li, C., Jegelka, S., Kohli, P.: Batched high-dimensional bayesian optimization via structural kernel learning. ArXiv  \textbf{abs/1703.01973} (2017)

\bibitem{Wang2016BayesianOI}
Wang, Z., Zoghi, M., Hutter, F., Matheson, D., de~Freitas, N.: Bayesian optimization in a billion dimensions via random embeddings. J. Artif. Intell. Res.  \textbf{55},  361--387 (2016)

\bibitem{wilson16}
Wilson, A.G., Hu, Z., Salakhutdinov, R., Xing, E.P.: Deep kernel learning. In: Proceedings of the 19th International Conference on Artificial Intelligence and Statistics (2016)

\bibitem{Xie2020}
Xie, Q., Dai, Z., Hovy, E.H., Luong, T., Le, Q.: Unsupervised data augmentation for consistency training. In: Advances in Neural Information Processing Systems 33, 2020 (2020)

\bibitem{10.1214/15-AOS1390}
Yang, Y., Dunson, D.B.: {Bayesian manifold regression}. The Annals of Statistics  \textbf{44}(2),  876 -- 905 (2016). \doi{10.1214/15-AOS1390}

\bibitem{ijcai2019p596}
Zhang, M., Li, H., Su, S.: High dimensional bayesian optimization via supervised dimension reduction. In: Proceedings of the Twenty-Eighth International Joint Conference on Artificial Intelligence, {IJCAI-19} (2019)

\bibitem{Ziomek2023}
Ziomek, J.K., Bou{-}Ammar, H.: Are random decompositions all we need in high dimensional bayesian optimisation? In: Krause, A., Brunskill, E., Cho, K., Engelhardt, B., Sabato, S., Scarlett, J. (eds.) International Conference on Machine Learning, {ICML} 2023, 23-29 July 2023, Honolulu, Hawaii, {USA}. Proceedings of Machine Learning Research, vol.~202, pp. 43347--43368. {PMLR} (2023)

\bibitem{https://doi.org/10.48550/arxiv.2111.02790}
Šehić, K., Gramfort, A., Salmon, J., Nardi, L.: Lassobench: A high-dimensional hyperparameter optimization benchmark suite for lasso (2021). \doi{10.48550/ARXIV.2111.02790}



\end{thebibliography}
\end{document}